\documentclass[letterpaper]{article} % DO NOT CHANGE THIS
\usepackage[preprint]{aaai2027}  % DO NOT CHANGE THIS
\usepackage[hyphens]{url}  % DO NOT CHANGE THIS
\usepackage{graphicx} % DO NOT CHANGE THIS
\usepackage{natbib}  % DO NOT CHANGE THIS AND DO NOT ADD ANY OPTIONS TO IT
\usepackage{caption} % DO NOT CHANGE THIS AND DO NOT ADD ANY OPTIONS TO IT
\usepackage[ruled, noend, noline]{algorithm2e}
\SetKw{Continue}{continue}
\SetKw{Break}{break}
\SetKw{Return}{return}
\SetKwProg{Fn}{Function}{ is}{end}

\usepackage{booktabs}

\usepackage{booktabs}
\usepackage{makecell}
\usepackage{amsmath,amssymb, amsfonts}
\DeclareMathOperator*{\argmin}{arg\,min}
\DeclareMathOperator*{\argmax}{arg\,max}
\usepackage{enumitem}
\usepackage{subcaption}
\usepackage{amsthm}
\theoremstyle{definition}
\newtheorem{definition}{Definition}
\theoremstyle{plain}

\newtheorem{theorem}{Theorem}

\usepackage{xspace}
\def\codt{\textsc{CODT}\xspace}
\def\streed{\textsc{STreeD}\xspace}
\def\contree{\textsc{ConTree}\xspace}

\def\rsoct{\textsc{RS-OCT}\xspace}
\def\cacontree{\textsc{CA-ConTree}\xspace}
\def\blossom{\textsc{Blossom}\xspace}
\def\quantbnb{\textsc{Quant-BnB}\xspace}
\def\gosdt{GOSDT\xspace}
\def\branches{\textsc{Branches}\xspace}
\def\numstrategies{18\xspace}

\title{Search Strategies for Optimal Classification and Regression Trees}
\author{
    Jacobus G. M. van der Linden,
    Mim van den Bos,
    Emir Demirovi\'{c} \\
}
\affiliations{
    Delft University of Technology\\
    J.G.M.vanderLinden@tudelft.nl, MimvandenBos@gmail.com, E.Demirovic@tudelft.nl
}

\begin{document}

\maketitle

\begin{abstract}
Optimal decision trees (ODTs) are compact, interpretable machine learning models that globally optimize a given objective, but their scalability  remains challenging.
While recent work has proposed a variety of search strategies to improve scalability, the precise contribution of each strategy remains unclear. To address this gap, we introduce a general algorithmic framework for ODTs that instantiates previously used search strategies and enables the definition of new ones. This provides a common lens through which to understand and compare different strategies, which we use to empirically investigate the effect of \numstrategies search strategies. Compared to the state of the art, the best strategy in our evaluation achieves significantly better anytime performance for classification, and improves runtime by more than an order of magnitude for regression.
\end{abstract}


\section{Introduction}

Applying machine learning (ML) in high-stakes domains, such as healthcare, demands models that are both accurate and interpretable \citep{rudin2019stop_black_box}. Decision trees are both: they are compact, easy-to-understand models that can capture nonlinear relationships in data.
Specifically, small decision trees are easy to understand \citep{piltaver2016comprehensible} and therefore, we explore optimal decision trees (ODTs), which globally optimize an objective that includes accuracy, constraints, and/or regularization terms on tree size. It has been shown that ODTs outperform traditional decision tree learning in both size and accuracy \citep{manual_linden2024opt_vs_greedy}.

However, learning ODTs is NP-hard \citep{laurent1976constructing}, and scaling to realistic dataset sizes and depths remains a central challenge. For example,
approaches based on integer programming \citep{bertsimas2017optimal} and Boolean satisfiability \citep{shati2023sat_odt_nonbin} can learn ODTs for a variety of objectives and constraints, but are typically limited to small datasets (under a thousand instances) and shallow depth limits (up to depth three).

This motivated the development of dedicated search algorithms for ODTs, which have achieved order-of-magnitude runtime improvements \citep{demirovic2022murtree}.
The number of related methods and search strategies has surged in the past few years, including depth-first search \citep{aglin2020learning, manual_brita2025contree}, best-first search \citep{lin2020generalized_sparse}, AND-OR search \citep{chaouki2025branches}, and novel techniques for optimal regression trees \citep{manual_bos2024srt}.
Further variants aim for anytime performance (solution quality over time) \citep{kiossou2026anytime_codt}.
Each of these methods comes with different strengths and weaknesses, offering trade-offs between scalability, support for numeric features, anytime performance, and support for problem settings beyond classification.

However, despite significant progress, each approach relies on its own assumptions and dedicated techniques, making comparisons difficult both conceptually and empirically. As a result, it is often unclear whether improvements stem from algorithmic ideas, implementation-level engineering, or the restriction to a particular problem setting.

The same variety also makes it challenging to combine methods. Transferring a technique from one approach to the other may require designing an entirely new algorithm from scratch, even when, at a high level, the two methods differ only slightly. For instance, a search strategy developed for classification could in principle be applied to regression, but such transfers are not trivial in practice. In other words, the research field lacks a clear, unified understanding of the fundamental differences between ODT algorithms.

We address this gap by devising an algorithmic framework for ODTs, which generalizes previous work and instantiates many algorithms from the literature, either exactly or up to small and practically negligible differences.

The main benefit of this framework is that it gives us a common lens through which to view methods that were previously studied in isolation. Within a single framework, we can directly compare a wide range of search strategies, both for classification and regression. 
This clarifies the difference between methods and allows us to draw (empirical) conclusions about the strengths and weaknesses of individual algorithmic components in a more principled way. 

In our experiments, we use our framework to empirically compare \numstrategies existing and new search strategies — to the best of our knowledge, the largest such evaluation for ODTs. For classification, the best strategy achieves better anytime performance than state-of-the-art ODT methods; for regression, it improves runtime by more than an order of magnitude.
 
\begin{table*}[t!]
    \centering
    \small
    \begin{tabular}{lllccccc}

    \toprule    
    
    Year &
    Authors &
    Method &
    Task &
    \makecell{Search\\strategy} &
    \makecell{Continuous\\features} &
    \makecell{Similarity\\ bound}&
    \makecell{D2\\solver} %&
    %Caching  
    \\ \midrule

    \citeyear{aglin2020learning} &
    \citeauthor{aglin2020learning} & 
    DL8.5 &
    C &
    DFS &
    &
    &
    %&
    %Branch
    \\
    
    &%\citeyear{lin2020generalized_sparse} &
    \citeauthor{lin2020generalized_sparse} & 
    GOSDT &
    C &
    BFS &
    &
    \checkmark &
    %&
    %Dataset
    \\

    \citeyear{demirovic2022murtree} &
    \citeauthor{demirovic2022murtree} & 
    MurTree &
    C &
    DFS &
    &
    \checkmark &
    \checkmark %&
    %Dataset
    \\

    &% \citeyear{mazumder2022odt_continuous_bnb} &
    \citeauthor{mazumder2022odt_continuous_bnb} & 
    Quant-BnB &
    C/R &
    DFS & 
    \checkmark &
    \checkmark &
     %&
    \checkmark %- 
    \\

    &%\citeyear{kiossou2022dl8.5_anytime} &
    \citeauthor{kiossou2022dl8.5_anytime} & 
    LDS-DL8.5 &
    C &
    LDS &
     &
     &
     %&
    %Branch
    \\

    \citeyear{zhang2022sparse_regtrees} &
    \citeauthor{zhang2022sparse_regtrees} & 
    OSRT &
    R &
    BFS &
     & 
     &
     %&
    %Dataset
    \\

    &%\citeyear{mendeley_linden2023streed} &
    \citeauthor{manual_linden2023streed} & 
    STreeD &
    C/R &
    DFS &
     &
    \checkmark &
    \checkmark  %&
    %Dataset 
    \\

    &%\citeyear{demirovic2023anytime_blossom} &
    \citeauthor{demirovic2023anytime_blossom} & 
    Blossom &
    C &
    DFS &
     &
     &
     %&
    %- 
    \\

    \citeyear{chaouki2025branches} &
    \citeauthor{chaouki2025branches} & 
    Branches &
    C &
    AND-OR &
     &
     &
     %&
    %Branch 
    \\

    &%\citeyear{brita2025contree} &
    \citeauthor{manual_brita2025contree} & 
    ConTree &
    C &
    DFS &
    \checkmark &
    \checkmark &
    \checkmark %&
    %Dataset 
    \\

    \citeyear{kiossou2026beamsearch_dt} &
    \citeauthor{kiossou2026beamsearch_dt} & 
    CA-DL8.5 &
    C &
    LDS &
     &
     &
     %&
    %Dataset 
    \\

     &
    \citeauthor{kiossou2026anytime_codt} & 
    CA-ConTree &
    C &
    LDS &
    \checkmark &
    \checkmark &
    \checkmark %&
    %Dataset 
    \\
    \midrule

     &
    This paper & 
    CODT &
    C/R &
    Many &
    \checkmark &
    \checkmark &
    \checkmark %&
    %- 
    \\
    
    \bottomrule        
    \end{tabular}
    \caption{Recent custom search-based methods for learning ODTs. Three observations: (i) most methods are limited to classification (C) and some can also do regression (R); (ii) most methods only work for binarized features; and (iii) a large variety of search strategies have been proposed, but never properly compared.}
    \label{tab:related_work_methods}
\end{table*}

\section{Related Work}

\paragraph{Greedy approaches.}
Learning ODTs is NP-hard \citep{laurent1976constructing, ordyniak2021fpt_odts}, so early methods, such as AID \citep{morgan1963aid} and CART \citep{breiman1984cart} relied on \emph{greedy top-down induction}.
These methods recursively construct a tree by selecting the split that optimizes for the current node an information criterion such as entropy or mean-squared error.
This approach is as fast as sorting the data, but it constructs trees that in theory can be arbitrarily larger than optimal \citep{garey1974gdt_performance_bounds}, and also empirically yields larger and less accurate trees than ODTs \citep{manual_linden2024opt_vs_greedy}.

\paragraph{Solver-based ODT approaches.} 
Greedy tree induction was later complemented by global optimization using mixed-integer programming \citep{bertsimas2017optimal} and Boolean satisfiability \citep{narodytska2018sat_odt}. Although subsequent work improved scalability through stronger formulations and more compact encodings \citep{verwer2017flexible, verwer2019learning, avellaneda2020efficient_sat_odt, shati2023sat_odt_nonbin, aghaei2024strong}, these methods still struggle with larger datasets and deeper trees.

\paragraph{Custom search-based ODT approaches.}
Better scalability has been achieved by custom search-based methods that incorporate techniques such as dynamic programming (DP).
DP approaches for ODTs \citep{schumacher1976dp_odt} obtain better scalability because they exploit the fact that subtrees can be solved independently. Recently, \citet{aglin2020learning} combined DP with branch and bound, \citet{lin2020generalized_sparse} proposed new similarity-based bounds, and \citet{demirovic2022murtree} contributed, among other improvements, a specialized subroutine for solving depth-two subtrees.
Most DP approaches assume numeric and nominal feature data is binarized beforehand and require a coarse binarization to remain scalable. \citet{mazumder2022odt_continuous_bnb} and \citet{manual_brita2025contree}, on the other hand, provide new pruning techniques that improve scalability for numeric features.

\paragraph{Search strategies.} 
A major difference among these methods is their search strategy. 
\citet{aglin2020learning}, \citet{demirovic2022murtree}, \citet{mazumder2022odt_continuous_bnb}, and \citet{manual_brita2025contree} use \emph{depth-first search} (DFS) which 
has a lower memory footprint, but can have a poor anytime performance if the search space is large. \citet{demirovic2023anytime_blossom} aim to improve the anytime performance of DFS by balancing the effort spent on optimizing the left and right subtrees.

\citet{lin2020generalized_sparse} propose \emph{best-first search} (BFS) by maintaining a global priority queue of open subproblems. They mention their search strategy as one of their key improvements over previous DFS approaches, since it prunes the search space more aggressively. However, BFS has a higher memory footprint, and later DFS approaches again outperformed their method, leaving the impact of their search strategy as an open question. Additionally, they do not address how to order the subproblems in the priority queue.
\citet{chaouki2025branches} instead frame the problem as an \emph{AND-OR search} (AOS) problem by maintaining a lower-bound based priority queue \emph{per search node}. They also attribute the performance improvements that they find to their search strategy. AOS, however, also has a high memory footprint.
Finally, \citet{kiossou2026beamsearch_dt} and \citet{kiossou2026anytime_codt} use \emph{limited discrepancy search} (LDS) to search around an increasingly large neighborhood of the greedy solution to improve anytime performance. 

To summarize, a clear understanding of the strengths and trade-offs of these search strategies is still lacking. We therefore propose a unified framework that generalizes most of the aforementioned search strategies, supports flexible objectives and numeric features, and enables a systematic comparison of their performance (see Table~\ref{tab:related_work_methods}).

\section{Preliminaries}

\paragraph{Notation.}
Let $\mathcal{D} = \{ (x^{(i)}, y^{(i)}) \}_{i=1}^n$ be a \emph{dataset} of $n$ \emph{samples} each of which is described by a \emph{feature vector} $x \in \mathbb{R}^p$ together with a label $y \in \mathcal{Y}$. We use $\mathcal{F} = \{1, \ldots, p \}$ to describe the set of features while $x_f^{(i)}$ is the value of feature $f$ in sample $i$. 
For classification $\mathcal{Y}$ is some finite set of size $k$ and for regression $\mathcal{Y} = \mathbb{R}$. 
Let $\mathcal{D}^f$ denote the sorted values $x_f$ for $(x,y) \in \mathcal{D}$ and let $\mathcal{U}^f$ be all unique sorted values in $\mathcal{D}^f$, then $S^f$ is the sorted list of possible thresholds, i.e., the midpoints between each pair of consecutive values in $\mathcal{U}^f$. We use $\tau_i$ to refer to the threshold in the list $S^f$ at index $i$. Finally, $\mathcal{D}(f \leq \tau)$ filters a dataset by the test $f \leq \tau$, i.e., $\mathcal{D}(f \leq \tau) =  \{ (x, y) \in \mathcal{D} \mid x_f \leq \tau \}$. 

A binary \emph{decision tree} $T$ is a recursive function $T : \mathbb{R}^p \rightarrow \mathcal{Y}$ that returns a label for a given feature vector. A decision tree can contain both \emph{branching nodes} and \emph{leaf nodes}.
Each branching node has two children and is described by a branching feature $f$ and a split threshold $\tau$, such that samples with $x_f  \leq \tau$ are passed to its left child while the rest is passed to the right.
Each leaf node is described by a label $\hat{y} \in \mathcal{Y}$. When computing $T(x)$, the feature vector $x$ is passed through the branching nodes beginning with the root node, until a leaf node is reached.
The label of that leaf node is the output of $T(x)$.
We use $|T|$ to denote the number of branching nodes in a decision tree, which we call its \emph{size}. The \emph{depth} of a decision tree is the maximum number of branching nodes on any root-leaf path.

\paragraph{Problem definition.} Given a dataset $\mathcal{D}$, a maximum depth budget $d$, and a regularization penalty $\lambda$, let $\mathcal{T}(\mathcal{D}, d)$ describe all possible decision trees for $\mathcal{D}$ with a maximum depth of $d$, then the optimal decision tree $T^*$ is defined as the tree that minimizes the loss plus a size penalty:
\[
    T^* = \argmin_{T \in \mathcal{T}(\mathcal{D}, d)} \sum_{(x, y) \in \mathcal{D}} \mathcal{L}\big(y, T(x)\big) + \lambda |T| \,,
\]
where $\mathcal{L}(\cdot, \cdot)$ is the loss function, e.g., for classification the zero-one loss, and for regression the squared error loss.

\paragraph{AND-OR search.}
AND-OR search generalizes shortest-path search, i.e., ``OR search''. In an OR graph, every node represents an \emph{either-or} choice among its outgoing edges, so a solution is a path in the graph. AND-OR graphs additionally contain AND nodes which require you to take \emph{all} of its outgoing edges. As a consequence, a solution is a subgraph rather than a single path.
Previously, \citet{martelli1973additive_and_or}, \citet{verhaeghe2020cp_odt}, and \citet{chaouki2025branches} presented ODT optimization as an AND-OR search problem: OR nodes correspond to choosing a feature test, while AND nodes correspond to evaluating both the left and the right subproblem of the chosen test. 

\begin{figure*}
    \centering
    \begin{subfigure}{0.35\textwidth}
        \includegraphics[width=0.9\textwidth]{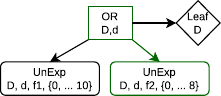}
        \caption{Select: based on the search strategy.}
    \end{subfigure}
    \hfill
    \begin{subfigure}{0.6\textwidth}
        \includegraphics[width=0.9\textwidth]{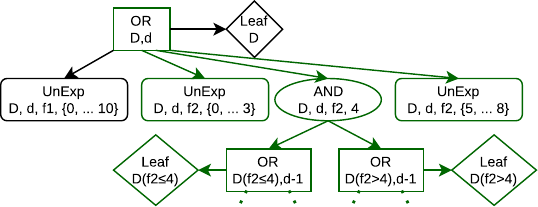}
        \caption{Expand: replace an unexpanded node with an AND node and two new unexpanded nodes for the remaining threshold intervals.}
    \end{subfigure}
    
    \caption{An example of (a) \textsc{Select} and (b) \textsc{Expand}. After expansion \textsc{BackPropagate} propagates the lower and upper bounds on the loss $\mathcal{L}$ from the newly expanded nodes back up to the root.}
    \label{fig:expand_select_backpropogate}
\end{figure*}

\section{Search Framework}
Our goal is to construct a framework for ODTs that supports a large variety of search strategies. To this end, we define an AND-OR search tree that is incrementally expanded in the order defined by the search strategies, and guarantees optimal solutions. Incremental expansion is necessary since materializing the whole search tree in memory is typically not possible. Bound-based pruning further helps to keep the search space small.
In the next section, we show that other approaches can be seen as instantiations of our framework.

\subsection{Defining the Search Tree}
We define an AND-OR search tree where the OR nodes decide the branching decisions (i.e., which feature test $f \leq \tau$ to split on), the AND nodes combine the left and the right subproblems, and the leaf nodes make the label assignment. Following \citet{manual_brita2025contree}, we include unexpanded (UnExp) nodes that store for a feature $f$ intervals of split indices $\{a, \ldots, b\} \subseteq \{ 1, \ldots, |S^f| \}$ such that all splits $f \leq \tau_i$ with $i \in \{a, \ldots, b\}$ are yet unexpanded. Formally:

\begin{definition}[Search tree]
\label{def:search_tree}
An ODT \emph{search tree} for the dataset $\mathcal{D}$ and depth limit $d$ is an AND-OR search tree $\mathcal{T}$ with as root an OR node with state $\langle \mathcal{D}, d \rangle$. Each OR node with state $\langle \mathcal{D}', d' \rangle$ in the tree $\mathcal{T}$ has the following children:
\begin{itemize}[itemsep=0pt, topsep=1pt]
    \item precisely one leaf node with state $\langle \mathcal{D}' \rangle$;
and, if $d > 0$:
    \item zero or more AND nodes with state $\langle \mathcal{D}', d', f, \tau \rangle$ with $f \in \mathcal{F}$ and $\tau \in S^f$; and
    \item zero or more UnExp nodes with state $\langle \mathcal{D}', d', f$, $\{ a, \ldots, b \} \rangle$ with $f \in \mathcal{F}$ and $\{ a, \ldots, b \} \subseteq \{ 1, ..., |S^f| \}$, provided that $
    |\{ a, \ldots, b \}| > 0$. 
\end{itemize}%
Each AND node with state $\langle \mathcal{D}, d, f, \tau \rangle$ has two children:
\begin{itemize}[itemsep=0pt, topsep=1pt]
    \item an OR node with state $\langle \mathcal{D}(f\leq \tau), d-1\rangle$; and
    \item an OR node with state $\langle \mathcal{D}(f > \tau), d-1\rangle$.
\end{itemize}%
UnExp nodes and leaf nodes have no children.
\end{definition}

The ODT is the subtree in $\mathcal{T}$ that minimizes the \emph{loss}.
\begin{definition}[Loss]
The \emph{loss} of a tree rooted in a node is:
    \begin{itemize}[topsep=1pt, itemsep=0pt]
        \item If it is a leaf with state $\langle \mathcal{D} \rangle$: 
        $\min_{\hat{y} \in \mathcal{Y}} \sum_{(x,y) \in \mathcal{D}} \mathcal{L}(y, \hat{y})$. 
        \item If it is an AND node with state $\langle \mathcal{D}, d, f, \tau \rangle$: the sum of the loss of its children plus the regularization penalty~$\lambda$.
        \item If it is an OR node with state $\langle \mathcal{D}, d \rangle$: The minimum of the loss of all its AND node children and its leaf child. 
    \end{itemize}
\end{definition}

For each leaf node, we store the label $\hat{y}$ that minimizes the loss, and in each OR node we store the split $f \leq \tau$ of the child AND node that minimizes the loss. When the search is complete, these values help construct the optimal decision tree from the search tree. The search is complete when the root node of the search tree is marked as \emph{complete}.

\begin{definition}[Complete] A node is \emph{complete} if it is expanded (OR, AND, or leaf) and all its children are complete.    
\end{definition}

The loss of incomplete nodes is unknown and therefore we track lower and upper bounds for it, and the incumbent loss (the best loss found so far for this node).
For complete nodes, the lower and upper bound are equal to the loss.

\begin{definition}[Lower bound]
\label{def:lower_bound}
    A \emph{lower bound} (LB) is any value less or equal to the loss of the node once it is complete.
\end{definition}%
Recent work has proposed a variety of lower bounds. In Appendix~\ref{app:backpropagation}, we describe which lower bounds we use.

\begin{definition}[Upper bound]
\label{def:upper_bound}
    An \emph{upper bound} (UB) for a node is a value such that any loss for this node larger than or equal to that value cannot yield an improved incumbent loss of the root node.
\end{definition}%
The incumbent loss of a node is a valid upper bound, as is the upper bound of its parent. A valid upper bound for an OR node is its parent AND node's upper bound minus its sibling's lower bound and the regularization penalty $\lambda$.

\subsection{Expanding the Graph}
We initialize the search tree with a single root OR node.
Its initial children are a leaf node, and one UnExp node for each feature $f \in \mathcal{F}$.
We then iteratively expand the graph by selecting and expanding an UnExp node.
After expansion, bounds on the loss and the incumbent loss are back-propagated along the path from the selected node to the root node. 
The full algorithm can therefore be described in terms of three procedures: \textsc{Select}, \textsc{Expand}, and \textsc{BackPropagate} (inspired by \citet{chaouki2025branches}), as shown in Fig.~\ref{fig:expand_select_backpropogate} and Alg.~\ref{alg:main}. 
The core difference between search strategies is how to select the next node to expand.

\setlength{\algomargin}{0.5em}
\begin{algorithm}[t!]
\caption{ODT search defined by the subprocedures \textsc{Select}, \textsc{Expand}, and \textsc{Back-Propagate}.}
\label{alg:main}
\DontPrintSemicolon
\Fn{$\operatorname{ODT}(\mathcal{D}, d)$} {
$\operatorname{root} \leftarrow \operatorname{Create-OR-Node}(\mathcal{D}, d)$\;
\While{$\operatorname{root}$ not  completed} {
$\operatorname{node}, \tau_i \leftarrow$ \textsc{Select}($\operatorname{root}$)\;
\textsc{Expand}($\operatorname{node}, \tau_i$) \;
\textsc{BackPropagate} ($\operatorname{node}$) \;
}
}
\end{algorithm}

\begin{definition}[Select]
\label{def:select}
The \textsc{Select} procedure is any function that for a given incomplete node $v$ returns any UnExp node with state $\langle \mathcal{D}, d, f, \{a, \ldots, b \} \rangle$ in the search tree rooted by $v$, and a threshold $\tau_i$ with $i \in \{a, \ldots, b \}$ to split on.
\end{definition}

In this work, we define \textsc{Select} recursively as shown in Alg.~\ref{alg:select} by distinguishing three cases: (i) how to select a child node for an OR node, (ii) how to select a child node for an AND node, and (iii) how to select a threshold for a given threshold interval. The three cases are decided by the subprocedures \textsc{SelectThreshold}, \textsc{SelectPriority}, and \textsc{SelectLeftRight} which are defined by each search strategy, which we discuss in the next section. 
We implement the selection within an OR node by maintaining a priority queue of all incomplete children, ordered by \textsc{SelectPriority}.

\begin{algorithm}[t!]
\caption{Selecting the next node to expand.}
\label{alg:select}
\DontPrintSemicolon
\Fn{\emph{\textsc{Select}}($\operatorname{node}$)} {
\uIf{$\operatorname{node}$ is unexpanded}{\Return $\operatorname{node}$, \textsc{SelectThreshold}($\operatorname{node}$)}
\uElseIf{$\operatorname{node}$ is an OR node} {
     $\operatorname{child} \leftarrow \displaystyle {\argmax_{\operatorname{child} \text{~of~} \operatorname{node}} \textsc{SelectPriority}(\operatorname{child})}$\;
     \Return \textsc{Select}($\operatorname{child}$)
}
\uElseIf{$\operatorname{node}$ is an AND node} {
     \Return \textsc{Select}(\textsc{SelectLeftRight}($\operatorname{node}$))
}
}
\end{algorithm}

\begin{definition}[Expand]
The \textsc{Expand} procedure takes an UnExp node $v$ with state $\langle \mathcal{D}, d, f, \{a, \ldots, b \} \rangle$ and a threshold $\tau_i$ to split on, with $i \in \{a, \ldots, b \}$, and replaces the node~$v$ with an AND node with state $\langle \mathcal{D}, d, f, \tau_i \rangle$ with its two OR children initialized in the same way as the root node. Two new UnExp nodes are added to the parent of $v$ with state $\langle \mathcal{D}, d, f, \{a, \ldots, i - 1 \} \rangle$ and $\langle \mathcal{D}, d, f, \{i +1, \ldots, b \} \rangle$, provided the threshold intervals are not empty.
\end{definition}
With the definitions for \textsc{Select} and \textsc{Expand} we can now state the first theorem that regardless of the search order, Alg.~\ref{alg:main} remains complete, which we prove in Appendix~\ref{app:proofs}.

\begin{theorem}[Completeness]
\label{th:complete}
Alg.~\ref{alg:main} (without the back-propagate step) results in a complete search tree for any choice of \textsc{Select} (according to Def.~\ref{def:select}).
\end{theorem}

After each expansion, Alg.~\ref{alg:main} does a back-propagation step: it propagates bounds, incumbent solutions, and prunes nodes that cannot yield an optimal solution. Formally:

\begin{definition}[Back-propagate]
\label{def:back_propagate}
    The \textsc{BackPropagate} procedure is any procedure that takes as input a node $v$, and for each node on the path from $v$ to the root, including the siblings of the OR nodes on that path, updates the incumbent and valid lower and upper bounds of the node based on its child nodes, and prunes any UnExp node for which its lower bound is larger than or equal to its upper bound.
\end{definition}%
Thus \textsc{BackPropagate} prunes and speeds up the search, but it never prunes the optimal solution, which we state in our final theorem and prove in Appendix~\ref{app:proofs}.

\begin{theorem}[Optimality preserving]
\label{th:optimal}
The \textsc{BackPropagate} procedure preserves optimality, i.e., it never prunes any UnExp node that after expansion would improve the incumbent of the root node.
\end{theorem}

Appendix~\ref{app:implementation} provides further details on the algorithm.

\section{Search Strategies}
We use our framework described above to instantiate  \numstrategies~search strategies, divided over the core categories of depth-first search (DFS), best-first search (BFS), limited-discrepancy search (LDS), and AND-OR search (AOS). For the sake of comparison, we frame all of these search strategies in the AOS paradigm explained in the previous section, such that each search strategy defines for each search node what the next best node to expand is. 
We describe the search strategies by defining the three components of the \textsc{Select} method: (i) the priority ordering in each OR node, (ii) the choice between the left and right child nodes of an AND node, and (iii) the selection of the threshold within a threshold interval. Table~\ref{tab:search_strategies} provides a summary. Unless otherwise stated, the selection of the threshold index in a threshold interval $\{ a, \ldots, b \}$ is the midpoint $\lfloor (a + b) / 2 \rfloor$.

\begin{table}[t!]
    \centering
    \small
    \begin{tabular}{lccccc}
    \toprule
    
    \textsc{SelectPriority} &  
    DFS & 
    BFS &
    LDS &
    AND-OR \\ \midrule

    Expanded &
    1st &
    2nd &
    2nd &
    2nd \\

    Lowest lower bound &
    2nd  &
     &
     &
    1st \\

    Lowest $\mathcal{H}$ &
     &
    1st & 
    1st & 
    \\ \midrule

    \textsc{SelectLeftRight} & 
    UB &
    $\mathcal{H}$ &
    $\mathcal{H}$ &
    UB \\    

    \midrule

    \textsc{SelectThreshold} & 
    $\lfloor \frac{a + b}{2}\rfloor$ &
    $\lfloor \frac{a + b}{2}\rfloor$ &
    $\mathcal{H}$ &
    $\lfloor \frac{a + b}{2}\rfloor$ \\    

    \bottomrule
    \end{tabular}
    \caption{The priority order by which each search strategy selects its next node to expand. $\mathcal{H}$ refers to a custom heuristic. }
    \label{tab:search_strategies}
\end{table}

\paragraph{Depth-first search.} DFS always continues expanding an expanded node until it is completed and then selects the next UnExp node to expand. We define the following variants.
\textsc{DFS-ConTree} imitates \contree \citep{manual_brita2025contree} by selecting first on feature rank (its ordering based on Gini impurity), then on the largest threshold interval ($b-a+1$ for threshold interval $\{a, \ldots, b \}$), and finally on the split index.
\textsc{DFS-Prio}, on the other hand, prioritizes nodes with the lowest lower bound and then proceeds in the order of \contree.
\textsc{DFS-Random} selects a random node to expand. 
The choice between the left and right child is based on the largest upper bound (largest potential for improvement), except for \textsc{DFS-Blossom}, a variant of \textsc{DFS-Prio} that imitates \blossom \citep{demirovic2023anytime_blossom} by choosing the left or right child with the fewest expansions.

\paragraph{Best-first search.} BFS sorts all open subproblems in a global priority queue based on a heuristic value $\mathcal{H}$.
We frame this in our AOS framework by making a node's priority equal to the minimum of its own and all its children heuristic value $\mathcal{H}$.
The heuristic value $\mathcal{H}$ is also used to choose between the left and right node of an AND node.
In Appendix~\ref{app:bfs_heuristics} we describe 11 candidates for the heuristic $\mathcal{H}$, most of which are a (linear) combination of a node's support (the size of node's dataset, $|\mathcal{D}|$) and its lower bound.
In the main text, we highlight \textsc{BFS-Big-LB} and \textsc{BFS-Small-LB}.
\citet{lin2020generalized_sparse} proposed \textsc{BFS-Big-LB} which prioritizes nodes with large support and small lower bounds.\footnote{See \url{https://github.com/ubc-systopia/gosdt-guesses}.} In contrast, \textsc{BFS-Small-LB} prioritizes search nodes with small support and a small lower bound, which our experiments in Appendix~\ref{app:search_strategies_experiments} show has the best performance.

\paragraph{Limited discrepancy search.} 
We consider LDS as a special case of BFS where, following \citet{kiossou2026anytime_codt}, the heuristic $\mathcal{H}$ is a combination of the feature and split discrepancy (the zero-based rank of a feature and of a split based on Gini impurity respectively). Let $\delta_f$ and $\delta_s$ be the feature and split discrepancy, then the search order for nodes with values $(\delta_f, \delta_s)$ is $(0,0), (1,0), (0,1), (2,0), (1,1), \ldots$, i.e., first prioritized by the lowest sum of the feature and split discrepancy, and then by the lowest split discrepancy.
The heuristic value of an unexpanded node $
\langle \mathcal{D}, d, f, \{a, \ldots, b \} \rangle$ is based on the best split discrepancy of the splits in the interval $\{ a, \ldots, b \}$. We try both setting the splitting threshold to the mid point of the interval (the default: \textsc{LDS-Mid}) and the split point with the lowest discrepancy (\textsc{LDS}), which imitates \cacontree \citep{kiossou2026anytime_codt}.

\paragraph{AND-OR search.} AOS prioritizes the search node with the lowest lower bound. If two nodes share the same lower bound, expanded nodes are prioritized. Unlike BFS, described above, AOS considers only the lower bound of the \emph{current} node, and not the lowest value in its subtree. AOS imitates \textsc{Branches} \citep{chaouki2025branches}.

\begin{figure}[t!]
    \centering
    \includegraphics[width=\columnwidth]{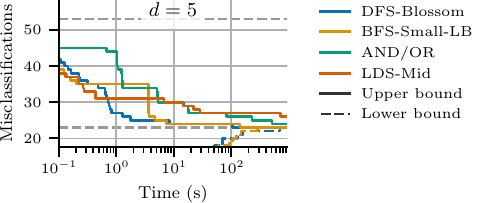}
    \caption{The anytime performance of four search strategy categories on the Wilt classification dataset ($d=5$).
    The solid colored lines are upper bounds (incumbent) and the colored dashed lines are lower bounds.
    The two gray dashed lines indicate the initial CART solution (top), and the best solution found by any strategy (bottom).
    }
    \label{fig:anytime_example}
\end{figure}
\begin{figure*}[t!]
    \centering
    \includegraphics[width=0.95\linewidth]{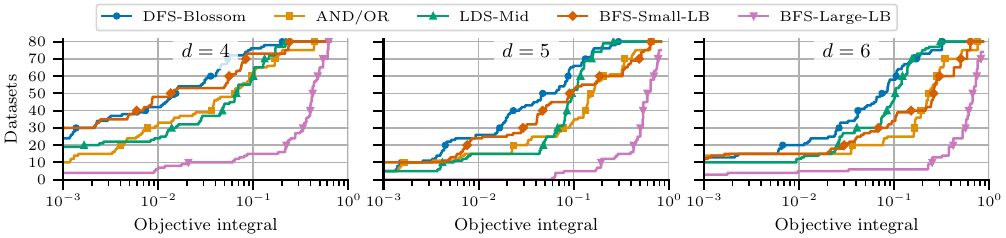}
    \caption{Cumulative distribution of the objective integral (OI) as a measure of anytime performance. The lines indicate for a given number on the horizontal axis what number of datasets are solved with an OI smaller than or equal to the given number. Lines that move quickly up indicate a better performance.}
    \label{fig:ecdf_ss_oi}
\end{figure*}
\begin{figure}[t!]
\centering
    \includegraphics[width=\linewidth]{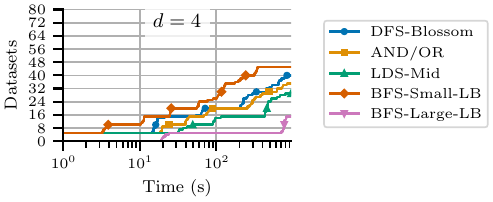}
    \caption{Cumulative distribution of runtime (s).}
    \label{fig:ecdf_ss_time}
\end{figure}

\section{Experiments}
We first compare the search strategies in our framework on runtime and anytime performance to identify the best approaches for different scenarios. We then compare the best strategy against state-of-the-art methods for classification and regression. Since all optimal methods find the same solutions, we omit out-of-sample comparisons and instead refer to \citet{mazumder2022odt_continuous_bnb}, \citet{manual_linden2024opt_vs_greedy}, and \citet{manual_brita2025contree} for previous comparisons with greedy learners. Appendix~\ref{app:more_experiments} provides additional results.

The results show that \textsc{BFS-Small-LB}, a best-first search that prioritizes nodes with small support and low lower bounds is best for completing the search, while \textsc{DFS-Blossom}, a depth first search that balances expansion of left and right children, is best for anytime performance. These search strategies also outperform the state of the art.

\subsection{Experiment Setup}
We ran the experiments on an Intel Xeon Gold 6448Y 32C 2.1GHz with eight cores and 32GB RAM running Linux Red Hat Enterprise 8.10 and repeated the experiments for each dataset five times.
For each run, we used a time-out of 15 minutes.
We implemented our framework \codt (Continuous Optimal Decision Trees) including all the aforementioned search strategies in Rust and provide a Python API.\footnote{The source code will be made public after paper acceptance.}
We describe the datasets and the experiment setup in more detail in Appendix~\ref{app:exp_setup}.

\paragraph{Metrics.} The main metrics by which we compare the search strategies and baselines are \emph{runtime} and the \emph{objective integral}. The runtime measures the time until the search is completed. The objective integral (OI) measures the quality of the incumbent solution over time. We assume that the incumbent is always initialized with the greedy heuristic solution with loss $\bar{s}$. Let $s^*$ be the optimal loss (or the best loss found by any method), let $s(t)$ be the loss of a method at time $t$, and let $\bar{t}$ be the time-out, then the OI measures:
\[
OI = \frac{1}{\bar{t}(\bar{s} - s^*)}\int_0^{\bar{t}} (s(t) - s^*)dt \,,
\]%
that is, the integral over time of the difference between the incumbent and the best solution, normalized to a value between $0$ and $1$ by dividing the integral by $\bar{t}(\bar{s} - s^*)$, the maximum obtainable value. Fig.~\ref{fig:anytime_example} represents visually what the OI measures: the surface area under the curve and between the two dashed lines that represent the CART solution and the best solution (found). A lower OI means finding good solutions earlier than a method with a high OI.
Additionally, in Appendix~\ref{app:memory} we report memory usage results.

\begin{figure*}[t!]
\centering
\includegraphics[width=\linewidth]{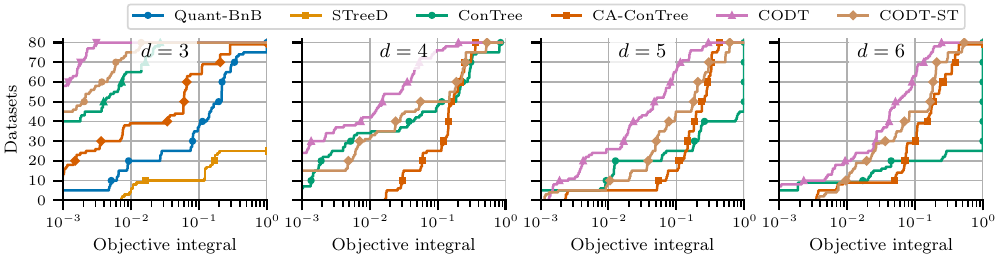}
\caption{Cumulative distribution of the objective integral (OI). 
For $d \leq 4$, \codt (ours) and \contree have the best anytime performance (lower OI). For $d \geq5$, \contree and \cacontree swap places but \codt still has the best performance.}
\label{fig:ecdf_others_oi}
\end{figure*}

\begin{figure}[t!]
    \centering
    \includegraphics[width=\linewidth]{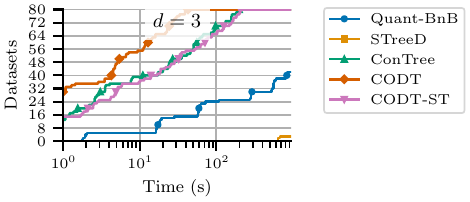}
    \caption{The cumulative distribution of runtime (s) for CODT and the classification baselines. \cacontree did not finish for any dataset within the given runtime limit. \codt's performance is close to \contree.}
    \label{fig:ecdf_others_time}
\end{figure}

\subsection{Comparing Search Strategies}
Here, we  present the results for the best search strategy within each category: \textsc{BFS-Small-LB}, \textsc{DFS-Blossom}, \textsc{LDS-Mid}, and \textsc{AND-OR}, and we include \textsc{BFS-Large-LB}, since this is the heuristic that \gosdt{} uses. The results for all search strategies are in Appendix~\ref{app:search_strategies_experiments}. 

Interestingly, Fig.~\ref{fig:ecdf_ss_time} shows that BFS with a heuristic that prioritizes nodes with small support and small lower bounds (\textsc{BFS-Small-LB}) has the best runtime. \textsc{BFS-Small-LB} solves open subproblems with low support first, which results in quick lower bounds that can be back-propagated in the search tree and help prune other solutions. 
In contrast, \textsc{BFS-Large-LB}, the heuristic used by \gosdt{} \citep{lin2020generalized_sparse}, is among the worst. 
Since \citeauthor{lin2020generalized_sparse} learn ODTs with a coarse binarization of numeric features, the search space in our experiments is orders of magnitude larger than in theirs, which may explain why their heuristic does not work well in our results, but \textsc{BFS-Small-LB} does.

For larger depth limits, almost all datasets result in a time-out for all search strategies, therefore we compare them by anytime performance (low OI).
Fig.~\ref{fig:ecdf_ss_oi} reveals \textsc{BFS-Small-LB} and \textsc{DFS-Blossom} as the best two approaches, while \textsc{LDS-Mid} and \textsc{AND-OR} show medium performance compared to the rest.
For $d\geq5$, \textsc{DFS-Blossom} performs significantly better than the rest. This highlights the importance of balancing expansion of left and right nodes, which has been neglected by most previous search strategies, but appears to be crucial for anytime performance.

These results motivate to use \textsc{DFS-Blossom} as \codt's search strategy for the rest of the experiments in this paper.

\subsection{Comparison with the Baselines}
\paragraph{Baselines.} We compare \codt with the following baselines: \gosdt, \quantbnb, \streed, \contree, \branches, and \cacontree (see Table~\ref{tab:related_work_methods}). 
\gosdt and \branches ran out of memory for all experiments and \blossom only supports binary classification, and therefore, we provide separate comparisons in Appendix~\ref{app:exp_other_baselines}.
We run \quantbnb only for $d=3$ because the implementation does not support larger depth limits.
Our implementation of \codt makes use of multi-threading, therefore, we also include the single-thread variant \textsc{CODT-ST} to compare \codt fairly with the single-threaded baselines. We describe the baselines and their configuration in Appendix~\ref{app:exp_setup}.

\paragraph{Results.}
Figs.~\ref{fig:ecdf_others_oi}-~\ref{fig:ecdf_reg_d3_time} show how our method \codt compares with the classification and regression baselines in runtime and objective integral.
\streed assumes binary features and struggles with the large number of splits that result from binarizing the numeric features.
Both in runtime and OI, \codt outperforms \quantbnb by a large margin, for both classification (110 times faster for $d=3$) and regression (40 times faster for $d=3$). 
In comparison to \contree, \codt performs similar in runtime performance, but \codt has a much better anytime performance for $d\geq 5$. 

Fig.~\ref{fig:ecdf_others_oi} shows that \citet{kiossou2026anytime_codt} improved the anytime performance of \contree for larger depth limits by iteratively searching increasingly large neighborhoods around the greedy solution. However, because it restricts expansion to splits near the greedy split, it cannot derive good lower bounds from solved subproblems and therefore struggles to prove optimality: it fails to solve any dataset within the time limit, including the smallest, Bank, which \codt and \contree solve in milliseconds for $d=3$.

In contrast,  our method \codt combines the best of both worlds and succeeds in both anytime performance and proving optimal solutions. Fig.~\ref{fig:ecdf_others_oi} shows it has the best anytime performance of all methods (statistically significant with a confidence level of 95\% according to a Friedman and Nemenyi test). Fig.~\ref{fig:ecdf_others_time} shows that \codt (single thread) has comparable runtime to \contree, and Fig.~\ref{fig:ecdf_reg_d3_time} shows it outperforms the state of the art for regression by a large margin. 

\begin{figure}[t!]
    \centering
    \includegraphics[width=\linewidth]{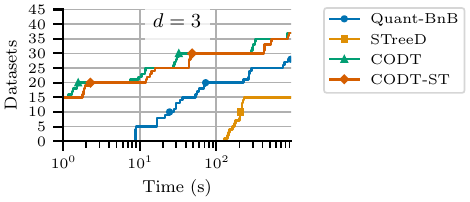}
    \caption{The cumulative distribution of runtime (s) for CODT and the regression baselines. \codt outperforms the state of the art by a large margin.}
    \label{fig:ecdf_reg_d3_time}
\end{figure}

\section{Conclusion}
We proposed a generalized framework for search strategies to learn optimal decision trees (ODTs) for both classification and regression.
The framework unifies many previously proposed search strategies, enabling a systematic comparison and the design of new strategies.
Our experiments demonstrate that two search strategies establish new state-of-the-art performance: a novel best-first strategy that (in contrast to previous work) prioritizes nodes with low support and strong lower bounds achieves the fastest optimality proofs, while a balanced depth-first strategy delivers the strongest anytime performance. These results show that the search strategy is a key determinant of ODT solver performance. Future work should investigate whether these insights extend to other AND-OR search problems.

\bibliography{references, manual_references}
\clearpage

\clearpage
\appendix
\setcounter{secnumdepth}{2}

\section{Detailed Method Description}
\label{app:implementation}
The main text already described the \textsc{Select} procedure. In this appendix, we explain the \textsc{Expand} and \textsc{BackPropagate} procedures, and provide other details.

\subsection{Expansion}
\label{app:expand}
Alg.~\ref{alg:expand} shows the pseudo-code for expansion given an unexpanded node for dataset $\mathcal{D}$, feature $f$, and open threshold interval $\{a, \ldots, b\}$, and a split point $\tau_i$.
The given node is first removed from its parent. Then a new AND node is created for the given split point. An AND node is initialized with two OR children with the data divided based on the split $f \leq \tau_i$.
Each OR node always contains a leaf node as a child.
If the remaining depth budget $d$ is larger than zero, then one unexpanded node is created for each numeric feature $f$ with the set of open thresholds set to $S^f$, the set of all thresholds $\tau$ that results in a unique split $f \leq \tau$ without an empty dataset on either side. 
After the AND node is created and added to the parent, two new unexpanded nodes are added to the parent with threshold intervals $\{a, \ldots, i-1\}$ and $\{i+1, \ldots, b \}$ provided the new intervals are not empty. This concludes the expansion step. See Fig.~\ref{fig:expand_select_backpropogate} for a visual example of an expansion step.

\begin{algorithm}[h!]
\caption{Expanding a node at split point $\tau_i$.}
\label{alg:expand}
\DontPrintSemicolon
\Fn{\emph{\textsc{Expand}}$(\operatorname{node}, \tau_i)$} {
$\langle \mathcal{D}, d, f, \{a, \ldots, b\} \rangle \leftarrow \operatorname{node.state}$ \;
$\operatorname{parent} \leftarrow \operatorname{node.parent}$ \;
$\operatorname{parent.remove}(\operatorname{node})$ \;
$\operatorname{child} \leftarrow \operatorname{Create-AND-Node}(\mathcal{D}, d, f, \tau_i)$\;
$\operatorname{parent.add}(\operatorname{child})$ \;
\uIf{$|\{a, \ldots, i-1\}| \geq 1$} {
    $\operatorname{left} \leftarrow \operatorname{Unexp}(\mathcal{D}, d, f, \{a, \ldots, i-1\})$\;
    $\operatorname{parent.add}(\operatorname{left})$ \;    
}
\uIf{$|\{i+1, \ldots, b\}| \geq 1$} {
    $\operatorname{right} \leftarrow \operatorname{Unexp}(\mathcal{D}, d, f, \{i+1, \ldots, b\})$\;
    $\operatorname{parent.add}(\operatorname{right})$ \;    
}
}
\Fn{$\operatorname{Create-AND-Node}(\mathcal{D}, d, f, \tau_i)$} {
    %$\mathcal{D}_L =  \{ (x, y) \in \mathcal{D} \mid x_f \leq \tau_i \}$ \;
    %$\mathcal{D}_R =  \{ (x, y) \in \mathcal{D} \mid x_f > \tau_i \}$ \;
    $\operatorname{left} \leftarrow \operatorname{Create-OR-Node}(\mathcal{D}(f \leq \tau_i), d-1)$ \;
    $\operatorname{right} \leftarrow \operatorname{Create-OR-Node}(\mathcal{D}(f > \tau_i), d-1)$ \;
    $\operatorname{node} \leftarrow \operatorname{AND-Node}(\mathcal{D}, d)$\;
    $\operatorname{node.add}(\operatorname{left})$\;
    $\operatorname{node.add}(\operatorname{right})$\;
    \Return $\operatorname{node}$
}
\Fn{$\operatorname{Create-OR-Node}(\mathcal{D}, d)$} {
    $\operatorname{node} \leftarrow \operatorname{OR-Node}(\mathcal{D}, d)$ \;
    $\operatorname{node.add}(\operatorname{Leaf}(\mathcal{D}))$\;
    \uIf{$d > 0$} {
        \For{$f \in \mathcal{F}$} {
            $\operatorname{child} \leftarrow \operatorname{Unexp}(\mathcal{D}, d, f, \{ 1, \ldots, |S^f| \})$\;
            $\operatorname{node.add}(\operatorname{child})$
        }
    }
    
    \Return $\operatorname{node}$
}
\end{algorithm}

\paragraph{Terminal case.} Previous work \citep{demirovic2022murtree, manual_brita2025contree} observed that major improvements can be obtained by replacing the terminal case $d=0$ with a specialized exhaustive procedure for when the remaining depth budget is one ($d=1$) or two ($d=2$). For example, for $d=1$, given a specific split $f \leq \tau_i$, we can find an optimal subtree in linear time \citep{manual_brita2025contree}, which is much faster than the normal procedure to create a search node for each possible child.

In this work, we also exploit this by providing a terminal case for $d = 2$. This means that whenever the algorithm reaches the function $\operatorname{Create-OR-Node}$ with $d=2$, instead of expanding the graph, we call a special subprocedure $\operatorname{Solve-D2}$ that immediately returns the optimal depth-two solution for this subproblem. $\operatorname{Solve-D2}$ is defined similarly to the main optimization loop (see Alg.~\ref{alg:main}), but during its expansion step, instead of creating a child AND node, it solves the $d=1$ subproblem exhaustively by using the algorithm from \citet{manual_brita2025contree}, which we generalize beyond classification.

Our generalization requires that the objective can be expressed as a function over per-sample cost tuples that are element-wise additive. For example, for minimizing misclassification, the cost tuple is the class count. For a problem with four classes ($\mathcal{K} = \{0,1,2,3\}$), a sample with label $1$ would result in cost tuple $\langle 0, 1, 0, 0\rangle$. These cost tuples are element-wise additive, so that we can compute the total class count for a dataset by adding the cost tuples for each sample. The final loss function takes a cost tuple $c$ and computes its cost:
\[
\operatorname{loss}_{\operatorname{misclassification}}(c) = \sum_i c_i - \max_i c_i \,.
\]%
Similarly, the optimal label is given by 
\[
\operatorname{label}_{\operatorname{misclassification}}(c) = \argmax_i c_i \,.
\]%

For regression, we follow \citet{manual_bos2024srt} by breaking down the costs into a three-tuple. For each sample with label $y$, we keep track of its label $y$, the square of its label $y^2$, and the sample count. So the initial cost tuple for a single sample is $\langle y, y^2, 1\rangle$. We can compute the sum of squared errors (SSE) as follows:
\[
\operatorname{loss}_{\operatorname{SSE}}(c) = c_1 - \frac{(c_0)^2}{c_2} \,,
\]%
which is the same as computing 
\[
\sum_{(x,y) \in \mathcal{D}} y^2 - \frac{(\sum_{(x,y) \in \mathcal{D}} y)^2}{|\mathcal{D}|} \,.
\]%
The optimal label is the mean
\[
\operatorname{label}_{\operatorname{SSE}}(c) = \frac{c_0}{c_2} \,.
\]%

With these defined, we can now describe the general version of the algorithm first proposed by \citet{manual_brita2025contree} in Alg.~\ref{alg:d1}, where $\operatorname{loss}(c)$, $\operatorname{label}(c)$, and $\operatorname{cost}(x,y)$ describe respectively the loss, label, and per-sample costs for a given optimization task.

\begin{algorithm}[t!]
\caption{Fully solving a depth-one node for dataset $\mathcal{D}$ and split $f \leq \tau$. Note that summations on $c$ are element-wise addition.}
\label{alg:d1}
\DontPrintSemicolon
\Fn{$\operatorname{Solve-D1}(\mathcal{D}, f, \tau)$} {
    $c_L, c_R \leftarrow \operatorname{zero-cost} \operatorname{tuple}$\;
    \For{$(x,y) \in \mathcal{D}$} {
        \uIf{$x_f \leq \tau$} {
            $c_L \leftarrow c_L + \operatorname{cost}(x,y)$\;
        }
        \uElse {
            $c_R \leftarrow c_R + \operatorname{cost}(x,y)$\;
        }
    }
    $\theta_L \leftarrow \operatorname{loss}(c_L), \theta_R \leftarrow \operatorname{loss}(c_R)$\;
    \For{$f_2 \in \mathcal{F}$} {
        $c_{LL}, c_{RL} \leftarrow \operatorname{zero-cost} \operatorname{tuple}$\;
        \For{$(x,y) \in \mathcal{D} \text{~sorted by~} f_2$} {
            \uIf{$x_f \leq \tau$} {
                $c_{LL} \leftarrow c_{LL} + \operatorname{cost}(x,y)$\;
                $\theta_L \leftarrow \min (\theta_L, 
                    \operatorname{loss}(c_{LL}) + \operatorname{loss}(c_L - c_{LL}) + \lambda)$ \;
            }
            \uElse {
                $c_{RL} \leftarrow c_{RL} + \operatorname{cost}(x,y)$\;
                $\theta_R \leftarrow \min (\theta_R, 
                    \operatorname{loss}(c_{RL}) + \operatorname{loss}(c_R - c_{RL}) + \lambda)$ \;
            }
            \lIf{$\theta_L \leq \lambda ~\wedge~ \theta_R \leq \lambda$} { \Break }
        }
        
    }
    \Return $\theta_L, \theta_R$
}
\end{algorithm}

In the first phase, we compute for the left and right subtree, split at $f \leq \tau$, what the cost tuple sum is for the data on the left ($c_L$) and the right of the split ($c_R$) are. These cost tuples are then used to compute an initial best loss $\theta_L$ and $\theta_R$ for the left and right subtree respectively. These solutions represent forming a leaf node on the left and right side of the given split.

In the second phase (the second for loop), we loop over all features and consider all possible splits on the second level of the subtree with that feature, i.e., the root node of the depth-two tree splits on $f \leq \tau$, and now we examine by a linear scan all possible splits in the two child nodes with feature $f_2$. For both the left and right subtree, we initialize the cost tuple on the left with the zero-cost tuple. Here, $c_{LL}$ represents the cost tuple of the left node in the left subtree and $c_{RL}$ represents the cost tuple of the left node in the right subtree. Then, one by one, sorted by $f_2$, we move one sample from the right node in the subtree to the left node by adding it to $c_{LL}$ (or $c_{RL}$ respectively). W.lo.g., the explanation continues for only the left subtree. We compute the loss over the samples that are now in the left node of the left subtree using $\operatorname{loss}(c_{LL})$. The loss over the remaining samples in the right node of the left subtree is obtained by subtracting the left node cost tuple from the left subtree's total: $c_L - c_{LL}$, for which we also compute the loss $\operatorname{loss}(c_L - c_{LL})$. Since this introduces a new split, we add the complexity cost $\lambda$ to get the total loss of the left subtree and compare this to the previous best left subtree cost. The same procedure also holds for the right subtree.

Finally, if at any time the costs for both subtrees are equal to or less than the complexity cost, we can finish the search early, since no improvement is possible anymore. The final best loss for the left and right subtree with root split $f \leq \tau$ is returned.

\subsection{Back-Propagation}
\label{app:backpropagation}
After expansion, the resulting solution and possibly new lower and upper bounds, and incumbent solutions are back-propagated along the path toward the root of the search tree. 

\paragraph{Updating the incumbent.}
For updating the incumbent, the procedure is straightforward. If the expansion resulted in a subtree with a lower loss than the incumbent, it is replaced. This is back-propagated along the path to the root by summing the best left subtree loss with the best right subtree loss and the complexity cost, and replacing the incumbent of each node on the path if the new loss is lower.

\paragraph{Lower bounds.}
For updating the lower bounds in each node (and for each possible split in that node), the procedure is more complicated.
We use the following sources for lower bounding:
\begin{description}
    \item[Trivial] For the loss functions considered in this work, zero is always a valid lower bound.
    \item[One-step lookahead] The optimal solution is either a leaf node or uses at least one split. Therefore, a valid lower bound is the minimum of the complexity cost $\lambda$ and the loss of the leaf node \citep{lin2020generalized_sparse}.
    \item[Clustering] The clustering lower bound considers the loss if any form of partitioning is allowed (ignoring the available splits). For regression, this is the $k$-means lower bound used by \citet{zhang2022sparse_regtrees} and \citet{manual_bos2024srt}. We adapt the bound for classification by assuming that the $k$ most frequent classes will be classified correctly. For both cases, we iterate over $k$ to compute the bound $k\lambda + \operatorname{clustering-loss}_k(\mathcal{D})$, with $k$ ranging from $1$ to the maximum number of remaining leaf nodes in the subtree.
    \item[Optimal solutions] An optimal solution to a subproblem is a valid lower bound.
    \item[Remaining splits] We keep track of the lower bound for each remaining unexpanded split interval. The minimum of these lower bounds is a valid lower bound for a node. In our implementation, we only use this bound if we know for a search strategy that the first node in an OR node's priority queue is guaranteed to be a valid lower bound. We only check the first node of the queue to prevent a costly scan over the whole queue.    
    \item[Similarity] As we further explain below, similarity between subproblems can be exploited to get lower bounds.
    \item[Propagation] The total loss of a subtree is the sum of the left and right subtree loss plus the complexity cost $\lambda$. Therefore lower bounds can be propagated to their parents.    
\end{description}

\paragraph{Similarity lower bounds and pruning.}
Similarity between subproblems can be exploited to prune the search space and compute lower bounds. We adopt (and adapt to regression) the three pruning techniques presented by \citet{manual_brita2025contree}, each of which are based on the \emph{similarity lower bound} by \citet{lin2020generalized_sparse} and \citet{demirovic2022murtree}.

The first two, \emph{neighborhood pruning} and \emph{interval shrinking}, reduce the size of open split threshold intervals by reasoning from similarity with an optimal solution for a specified threshold. E.g., if the optimal loss with root feature test $f \leq \tau_i$ is equal to $s$, the current incumbent solution has value $u$, then any replacement of the incumbent needs to improve on $s$ by at least $\delta = s - u$. For minimizing misclassification score, this means that any threshold $\tau_a < \tau_i$ or $\tau_b > \tau_i$ can only lead to an improving solution if at least $\delta$ samples switch from the left to the right subtree (or vice versa) with the new split threshold. This insight is exploited to shrink the intervals on the left and right of $\tau_i$ from $\tau_{i-1}$ to $\tau_a$ and from $\tau_{i+1}$ to $\tau_b$

The third pruning technique, \emph{sub-interval pruning} (see also \citet{mazumder2022odt_continuous_bnb}), reasons about two solved subproblems, one for $\tau_a$ and another for $\tau_b$, such that $\tau_a < \tau_b$. Let $s_L(a)$ be the loss of the left subtree with root $f \leq \tau_a$ and $s_R(b)$ the loss of the right subtree with root $f \leq \tau_b$. Then, if $s_L(a) + s_R(b) \geq u$, with $u$ the incumbent solution for this subtree, then any threshold $\tau_i$ with $\tau_a < \tau_i < \tau_b$ cannot lead to an improving solution, and thus any interval within these two thresholds should be pruned. This pruning technique applies for both classification and regression.

\paragraph{When to prune.}
We could prune unexpanded nodes from the queue as soon as its lower bound exceeds the incumbent. However, when we prune a node as soon as its lower bound equals the incumbent, we get no useful information for pruning the neighboring splits. If instead, we would have searched longer and found a higher lower bound, we could have pruned more. 

Therefore, we adapt the search and pruning strategy. Rather than immediately pruning nodes that are proven to be suboptimal, we continue the search, and only prune a node if it cannot further improve the bound of any other neighboring node (nodes with a neighboring split threshold).
This is the case when the entire remaining interval can be pruned or when the optimal solution has been found.

\subsection{Implementation Details}
Finally, this section provides some last details on our implementation.

\paragraph{Cache.} Unlike most previous approaches to optimal decision tree search, we do not use caching. Therefore, \codt is better described as a divide and conquer approach with branch-and-bound then dynamic programming. 
Previous dynamic programming methods were typically developed under the assumption of a coarse binarization, which resulted in frequent occurrence of repeated subproblems. However, given our numeric input data and pruning techniques, we observe very few reoccurring subproblems. However, if \codt would be applied to data where subproblems are repeated frequently, a cache could easily be added.

\paragraph{Multi-threading.} To speed up computation, we add support for multi-threading.
We create $|\mathcal{F}|$ independent search trees for each of the $|\mathcal{F}|$ initial children of the root OR node, i.e., for each $f \in \mathcal{F}$  we create a search tree with an OR root node that only has one unexpanded child, the unexpanded node with state $\langle \mathcal{D}, d, f, \{ 1, \ldots, |S^f| \} \rangle$. The selected search strategy is then independently applied in a separate thread for each search tree. The final solution is the best solution among all threads. Incumbent solutions for the root node are shared with other threads. The search is over when all threads have closed.

\section{Search Strategies}
\label{app:extended_search_strategies}
In this appendix, we provide a full description of all search strategies evaluated in this paper. We group the search strategies into three main categories: DFS, BFS (including LDS) and AND-OR. For AND-OR we only consider one variant, and therefore no further details need to be given.

For all search strategies, unless otherwise specified, \textsc{SelectThreshold} for a threshold interval $\{a, \ldots, b \}$ selects the midpoint $\lfloor (a+b)/2 \rfloor$.

\begin{table*}[t!]
    \centering
    \begin{tabular}{lcl}
    \toprule
         Name &  
         Definition of $\mathcal{H}$ &
         Intuition \\ \midrule

          \textsc{BFS-Curiosity} &
         $\frac{LB}{|\mathcal{D}|}$ &
         Lowest minimum error first, weighted by impact on the root \\

         \textsc{BFS-LB} &
         $LB$ &
         Lowest minimum error first \\

         \textsc{BFS-Small} &
         $|\mathcal{D}|$ &
         Small nodes first for quick bounds \\

         \textsc{BFS-Large} &
         $-|\mathcal{D}|$ &
         Large nodes first for big improvements in bounds \\
         
         \textsc{BFS-Large-LB} &
         $LB - |\mathcal{D}|$ &
         Most samples that may still be classified correctly (what GOSDT uses) \\
         
         \textsc{BFS-Small-LB} &
         $LB + |\mathcal{D}|$ &
         Small nodes first for quick bounds, but also consider lower bounds \\

         \textsc{BFS-LB-t-Large} &
         $LB - \varepsilon \cdot |\mathcal{D}|$ &
         Lowest minimum error first, break ties with larger nodes\\
         
         \textsc{BFS-LB-t-Small} &
         $LB + \varepsilon \cdot |\mathcal{D}|$ &
         Lowest minimum error first, break ties with smaller nodes \\

        \textsc{BFS-Large-t-LB} &
         $-|\mathcal{D}| + \varepsilon \cdot LB$ &
         Large nodes first with lower bound as tie breaker\\
         
         \textsc{BFS-LB-t-Small} &
         $|\mathcal{D}| + \varepsilon \cdot LB$ &
         Small nodes first with lower bound as tie breaker \\

    \bottomrule
    \end{tabular}
    \caption{An overview of the heuristics that combine support and lower bound. Here $\varepsilon = 1\times10^{-6}$ turns the term that follow into a tie breaker.}
    \label{tab:bfs_heuristics}
\end{table*}
\subsection{DFS Variants}
The main idea of depth-first search (DFS) is to continue expanding a subtree of the search tree until it is fully expanded, and only then to continue expanding another search node. Therefore, the first selection criterion for DFS in the priority order is to select nodes that are expanded. The remaining priority ordering differs per DFS variant. Based on \citet{demirovic2022murtree}, the default selection of the left or right child in an AND node is to select the child with the highest upper bound (initially from the leaf node solution).

We include the four variants of DFS in our analysis, the following explains how each variant prioritizes nodes in an OR node whenever no expanded node is in the priority queue anymore, and how the left or right node is selected in AND nodes if this deviates from the default.

\begin{description}
    \item[\textsc{DFS-Random}] selects the next node to expand in each OR node randomly. The choice between the left and right subtree in an AND node is also made randomly. This variant is included as a baseline to compare the other methods against.
    \item[\textsc{DFS-ConTree}] simulates how \contree searches. Its second priority criterion is a feature's rank (the ranking of the feature based on Gini impurity). Its third priority criterion is the size of the threshold interval (larger first). Its fourth priority criterion is the start point of the interval (lower start point first).
    \item[\textsc{DFS-Prio}] only differs from \textsc{DFS-ConTree} in one point. Its second priority criterion is a node's lower bound. The remaining priority criterion follows the same order: feature rank, size of the threshold interval, start of the threshold interval.
    \item[\textsc{DFS-Blossom}] adapts \textsc{DFS-Prio} in one way: the left or right child in an AND node is not selected by the highest upper bound, but by the fewest number of previous expansions. This ensures that the left and right side are explored in a balanced way and helps improve anytime performance.
\end{description}
In the experiments in Appendix~\ref{app:more_experiments} we test the effect of these variants.

\subsection{BFS Variants}
\label{app:bfs_heuristics}
We compare an even larger number of variants of best-first search (BFS). Each variant defines the same ordering based on a heuristic value $\mathcal{H}$, but differs in what that heuristic value is. The general priority order in an OR node is: first, the node with the lowest value for $\mathcal{H}$; second, nodes that are already expanded; third, the node with the lowest feature rank; and fourth, the node with the lowest interval start point. The selection of the left or right child in an AND node is also based on the heuristic value $\mathcal{H}$.

\codt computes the value $\mathcal{H}$ for each unexpanded node. For each expanded node, the value $\mathcal{H}$ is the minimum of the value of all its children (recursively). While \codt selects a node in multiple steps using different priority queues in each search node, giving each expanded node the minimal heuristic value of all its children, in effect creates a global priority queue with only the unexpanded nodes in it.

The BFS heuristics considered here can be grouped into three categories: (i) based on a random number, (ii) based on a combination of support and lower bound, and (iii) based on the feature and split discrepancy (the LDS variant).

\begin{description}
    \item[Random] \textsc{BFS-Random} computes a random value for each unexpanded node. This search strategy is included as a baseline to see if other heuristics provide performance better or worse than random.
    \item[Combination of support and lower bound] Many variants of combining a node's support ($|\mathcal{D}|$) and a node's lower bound can be turned into a heuristic. In the literature, we already find two. \citet{hu2019sparse} define the curiosity heuristic $\mathcal{H} = LB / |\mathcal{D}|$. \citet{lin2020generalized_sparse} do not discuss the heuristic in their paper, but their code contains the heuristic $\mathcal{H} = LB - |\mathcal{D}|$.\footnote{See \url{https://github.com/ubc-systopia/gosdt-guesses}. Note that they write it as $|\mathcal{D}| - LB$ since they choose the node with the highest heuristic value, whereas we prioritize the node with the lowest heuristic value.} In our work, we test a number of variants, each of which can be written as a linear combination of a node's lower bound and support: $\mathcal{H} = a\cdot LB + b\cdot |\mathcal{D}|$, with $a$ and $b$ some constants. For $a$ and $b$ we consider the values $0$, to not include either the lower bound or support in the heuristic, or $1$ to include it, or $1 \times 10^{-6}$ to include it but only as a tie-breaker. In Table~\ref{tab:bfs_heuristics} we provide an overview of the heuristics included in our experiments.
    \item[Limited discrepancy] We treat LDS as a special case of BFS which prioritizes nodes that are close to the greedy solution. We follow the search order by \citet{kiossou2026anytime_codt} who do a diagonal sweep over the feature rank and split discrepancy. Formally, let $\langle r, s, d \rangle$ represent an unexpanded node's feature rank, split discrepancy, and depth budget respectively, then we use $\mathcal{H} = M \cdot (r + s) + s - m\cdot d$ as the heuristic value, with $M$ some large number and $m$ some small value. For fixed $d$, this ensures that nodes are first sorted by the sum $r + s$ and then by $s$. This ensures the order of $r,s$ suggested by \citeauthor{kiossou2026anytime_codt}: (0,0), (1,0), (0,1), (2,0), (1,1), (0,2), \ldots . Additionally, we add the depth budget $d$ to the heuristic, as a tie breaker, so that if the other values are the same, the heuristic will prefer splits closer to the root. The feature rank $r$ is computed by finding the best split for each feature based on a local Gini impurity metric (as CART does). The split with the lowest Gini impurity gets rank 0, the second best gets rank 1, etc. Similarly, given a feature $f$, we rank each possible split $\tau$ based on a local Gini impurity metric too. The best $\tau$ gets rank 0, the second best rank 1, etc. This is captured by the split discrepancy $s$. Specifically, for a range of thresholds $\tau_a,\ldots \tau_b$, the split discrepancy of this unexpanded node is equal to the minimum of the split discrepancies of the splits in that range.

    We consider two variants of LDS that only differ in how the split threshold from a range of thresholds is chosen. The first, \textsc{LDS-Mid} follows the default by selecting the midpoint $\lfloor (a+b)/2 \rfloor$. The second, \textsc{LDS} chooses the split that minimizes the split discrepancy.
\end{description}%
When BFS is at risk of hitting a memory limit, \codt dynamically switches back to \textsc{DFS-Prio} as its strategy for selecting a node to expand in OR nodes. If hitting the memory limit is no longer a risk, \codt switches back to BFS.

\section{Proofs}
\label{app:proofs}
In this appendix, we prove the two theorems in the main text.

\subsection{Proof of Completeness}
Theorem~\ref{th:complete} states: Alg.~\ref{alg:main} (without the back-propagate step) results in a complete search tree for any
choice of \textsc{Select} (according to Def.~\ref{def:select}). We prove this theorem by induction.
\begin{proof}
    Let $P(d)$ be the proposition that after running Alg.~\ref{alg:main} (without the back-propagate step), any OR node in the search graph with state $\langle \mathcal{D}, d \rangle$ for any dataset $\mathcal{D}$ is complete.

    \emph{Base case: $d=0$.} Any OR node with state $\langle \mathcal{D}, d \rangle$ with $d=0$ has only one child: the leaf node with state $\langle \mathcal{D} \rangle$, and therefore the OR node is complete. Therefore, $P(0)$ holds.

    \emph{Recursive case.} Assume that $P(d-1)$ with $d \geq 1$ holds. Now any OR node with state $\langle \mathcal{D}, d\rangle$ will eventually be completed by Alg.~\ref{alg:main} if all its children will be completed. It can have the following children:
    \begin{itemize}
        \item A leaf node with state $\langle \mathcal{D} \rangle$: a leaf node is always complete.
        \item An AND node with state $\langle \mathcal{D} , d, f, \tau \rangle$ with two OR children with state $\langle \mathcal{D}(f \leq \tau), d-1\rangle$ and $\langle \mathcal{D}(f > \tau), d-1\rangle$: since we assume that $P(d-1)$ holds, the AND node will be completed by Alg.~\ref{alg:main}. 

        \item An UnExp node with state $\langle \mathcal{D}, d, f, \{a, \ldots, b
        \} \rangle$: Alg.~\ref{alg:main} continues until the root node is complete, and the root is not complete as long as it has at an incomplete child in its subtree. Since $\mathcal{F}$ are $S^f$ finite sets and $d$ is a finite number, the fully materialized search tree is also finite. While each expansion of an UnExp nodes creates two new UnExp nodes, the combined sizes of the threshold intervals $\{ a, \ldots, i-1\}$ and $\{i+1, \ldots, b \}$ is always strictly smaller than the size of the original UnExp node interval $\{a, \ldots, b \}$. The creation of new UnExp nodes will stop when $|\{a, \ldots, i-1
        \}| = 0$ and $|\{i+1, \ldots, b
        \}| = 0$. Therefore, expanding an interval $\{a, \ldots, b\}$, and then expanding the resulting smaller intervals, will eventually result in creating precisely $b-a +1$ AND nodes, each of which will be completed (see above). Since Alg.~\ref{alg:main} does not stop until the root is complete, and the root is not complete while there is an UnExp node in its subtree, eventually Alg.~\ref{alg:main} will replace all UnExp nodes with AND nodes for each possible split.
    \end{itemize}
    Therefore, all children of an OR node will be completed by Alg.~\ref{alg:main} or replaced with nodes that will be completed. Therefore, $P(d)$ also holds.

    Since both the base case and the recursive case hold regardless of the choice of the \textsc{Select} procedure, Theorem~\ref{th:complete} holds.
\end{proof}

\subsection{Proof of Optimality Preservation}
Theorem~\ref{th:optimal} states: The \textsc{BackPropagate} procedure preserves optimality, i.e., it never prunes any UnExp node that after expansion would improve the incumbent of the root node.

\begin{proof}
    This theorem follows directly from the definitions of the lower and upper bounds (Defs.~\ref{def:lower_bound} and~\ref{def:upper_bound}) and the definition of \textsc{BackPropagate} (Def.~\ref{def:back_propagate}).
    By definition \textsc{BackPropagate} only prunes UnExp nodes that have a lower bound $LB$ larger than or equal to the upper bound $UB$, i.e., $LB \geq UB$. Since a valid lower bound is a bound on the best loss obtainable after expansions from this node (or its replacements), and a valid upper bound $UB$ indicates that any value equal to or larger than $UB$ cannot result in an improvements of the incumbent of the root, when $LB \geq UB$ holds for a node, then it is impossible that the node, after expansions, can result in an improvement of the root incumbent. Therefore it can be pruned.
\end{proof}

\begin{table}[t!]
    \centering
    \begin{tabular}{lcccc}
    \toprule
        Dataset &
        $|\mathcal{D}|$ & 
        $|\mathcal{F}|$ &
        $\sum_f |U^f|$ &
        $|\mathcal{Y}|$ \\ \midrule

Avila & 20867 & 10 & 41110 & 12 \\
Bank & 1372 & 4 & 5016 & 2 \\
Bean & 13611 & 16 & 211966 & 7 \\
Bidding & 6321 & 9 & 12528 & 2 \\
Eeg & 14980 & 14 & 5404 & 2 \\
Fault & 1941 & 27 & 19286 & 7 \\
Htru & 17898 & 8 & 124959 & 2 \\
Magic & 19020 & 10 & 147097 & 2 \\
Occupancy & 20560 & 5 & 19721 & 2 \\
Page & 5473 & 10 & 9082 & 5 \\
Raisin & 900 & 7 & 6289 & 2 \\
Rice & 3810 & 7 & 24635 & 2 \\
Room & 10129 & 16 & 3072 & 4 \\
Segment & 2310 & 18 & 14910 & 7 \\
Skin & 245057 & 3 & 765 & 2 \\
Wilt & 4839 & 5 & 22599 & 2 \\
\midrule
Casp & 45730 & 9 & 298346 & 15903 \\
Concrete & 1030 & 8 & 1517 & 845 \\
Energy & 19735 & 28 & 107798 & 92 \\
Fish & 908 & 6 & 1810 & 827 \\
Gas & 36733 & 10 & 142100 & 4447 \\
Grid & 10000 & 12 & 119988 & 10000 \\
News & 39644 & 59 & 558454 & 1454 \\
Qsar & 546 & 8 & 1605 & 515 \\
Query1 & 10000 & 3 & 29997 & 539 \\

    \bottomrule
    \end{tabular}
    \caption{The datasets used in our experiments. The values are the number of samples, the number of numeric features, the number of unique feature values (from which we obtain all possible splits), and the number of unique class labels.}
    \label{tab:data}
\end{table}

\section{Experiment Setup Details}
\label{app:exp_setup}

\subsection{Experiment Data}
In our experiments we use the same datasets as those used by \citet{mazumder2022odt_continuous_bnb} and obtained them from their repository.\footnote{\url{https://github.com/mengxianglgal/Quant-BnB}}
These datasets are originally from the UCI Machine Learning Repository \citep{uci2017} and are preprocessed by removing features that do not aid in prediction, such as unique identifiers of samples and timestamps of recording.
We do not include the Carbon and Query2 datasets because they are multivariate regression datasets and we limit our experiments to univariate regression.
For all our experiments, we use the full dataset, i.e., both the train and test data as one dataset. Table~\ref{tab:data} shows an overview of the datasets, mentioning the number of samples, the number of numeric features, and the sum of the unique values per feature (a value very close to the total number of possible splits), and the number of unique label values.

For the baseline methods that cannot handle numeric features directly, i.e., \gosdt, \blossom, \streed, and \branches, we binarize the features in such a way that these methods solve precisely the same problem: we create one binary feature for every possible split on the numeric features.

\subsection{Configuration of the Baselines}
This section describes the baselines used in our experiments and their configuration. Since not all baselines support setting the regularization penalty $\lambda$, we default to $\lambda = 0$ in all experiments.

\begin{description}
    \item[\gosdt]\citep{lin2020generalized_sparse} \gosdt is a dynamic programming with branch-and-bound approach for solving ODTs.
    It uses a best-first search strategy with a global priority queue.
    It requires binarization of the input features, but implements a similarity-based pruning method to handle numeric features more efficiently. \gosdt can optimize a variety of objectives that can be expressed as a function of the false positives and negatives. The current version of \gosdt includes improvements proposed by \citet{mctavish2022sparse_trees_guess_ensembles} and it is this version that we compare with.\footnote{\url{https://github.com/ubc-systopia/gosdt-guesses}, v1.0.4} 
    We use the default parameters, except that we allow for a small regularization parameter $\lambda$, since we set $\lambda = 0$ in our experiments. Since \gosdt runs out of memory in all our experiments, we compare with it on subsampled datasets in Appendix~\ref{app:exp_other_baselines}.
    \item[\quantbnb]\citep{mazumder2022odt_continuous_bnb} \quantbnb is a branch-and-bound approach for solving ODTs, both for classification and regression.\footnote{\url{https://github.com/mengxianglgal/Quant-BnB}, 0a381a5}
    It is a depth-first search strategy which for each optimization depth first consider as splits only $s$ quantiles on the numeric features. If no proof of optimality is found yet, it searches for another $s$ quantiles between the previous quantiles, and repeats this process until all possible splits are considered. We use the recommended parameters: we use the $W_{1,s'}$ lower bound and set $s = 3$.
    The implementation provided by the authors does not provide a recursive formulation but has two functions, one for optimizing depth-two and another for depth-three trees. Therefore, we do not run \quantbnb to optimize trees beyond depth three.    
    \item[\streed]\citep{manual_linden2023streed} \streed is a depth-first search dynamic programming with bounds approach for learning ODTs that supports optimization of a large variety of objectives and constraints.\footnote{\url{https://github.com/AlgTUDelft/pystreed}, v.1.3.9}
    \streed expects binary features and therefore we use its internal binarizer and create a binary feature for each possible split by setting the parameter \texttt{num\_thresholds} equal to the dataset size. For some datasets this step alone, however, results in exceeding the memory limit. Moreover, \streed implements the same depth-two subroutine as MurTree \citep{demirovic2022murtree} which creates a table whose size is quadratic in the number of binary features. For small numbers of binary features (less than 1000 for example) this technique provides a major speed-up, but in our experiments, it may exceed the memory limit or the algorithm never finishes iterating over each cell in the table before the time-out. Therefore, we see in the experiments that \streed for many datasets does not have a result, showing the importance of having dedicated techniques for handling numeric features.
    Support for regression was added to \streed by \citet{manual_bos2024srt}. For classification, we run the optimization task \texttt{cost\_complex\_accuracy} and for regression we use \texttt{cost\_complex\_regression}. For the rest, we use the default parameters.
    \item[\blossom]\citep{demirovic2023anytime_blossom}
    \blossom is an adaption of previous DP approaches that aims at anytime performance.\footnote{\url{https://gitlab.laas.fr/ehebrard/blossom}, cf7eec31}
    However, it is does not utilize a cache. It uses a depth-first search strategy, but makes sure that both the left and the right tree have solutions before proceeding with optimizing only one side. The provided implementation only supports binary classification, and therefore we left it out of the experiments in the main text (where we also include multi-class classification datasets), but report results on only the binary classification separately in Appendix~\ref{app:exp_other_baselines}. We run \blossom using the default parameters.
    
    \item[\contree]\citep{manual_brita2025contree} \contree is a depth-first search algorithm that for each search node keeps track of intervals of split thresholds that may still lead to an improvement of the incumbent.\footnote{\url{https://github.com/ConSol-Lab/contree}, v1.0.8}
    We run \contree with the default parameters.
    \item[\branches]\citep{chaouki2025branches} \branches models ODT learning as and AND-OR search problem, similar to us.\footnote{\url{https://github.com/Chaoukia/branches}, commit 2937832}. The current implementation is in Python, and therefore not as optimized for runtime performance as the other baselines. \branches expects binary or categorical features, so we use the same binarization as for \streed.
    Since \branches runs out of memory in all our experiments, we compare with it on subsampled datasets in Appendix~\ref{app:exp_other_baselines}.
    We run \branches with the default parameters.
    
    \item[\cacontree]\citep{kiossou2026anytime_codt} \cacontree is a limited-discrepancy search adaption of \contree.\footnote{We used the code shared in the link in the original arXiv paper: \url{https://anonymous.4open.science/r/contree-rs-C7B8}.}
    It runs \contree as a subprocedure while restricting the splits allowed to make. In the first iteration, only the original greedy splits can be considered. In the second iteration, also all the splits that have a discrepancy of one with respect to the greedy splits, etc... We set the minimum support to one, and use the parameters \texttt{--fast-d2}, \texttt{--split-selection-strategy first}, and \texttt{--use-lds}, \texttt{--sort-by-heuristic}.
\end{description}

\begin{figure*}[t!]
    \centering
    \begin{subfigure}{\textwidth}
        \centering
        \includegraphics[width=\textwidth]{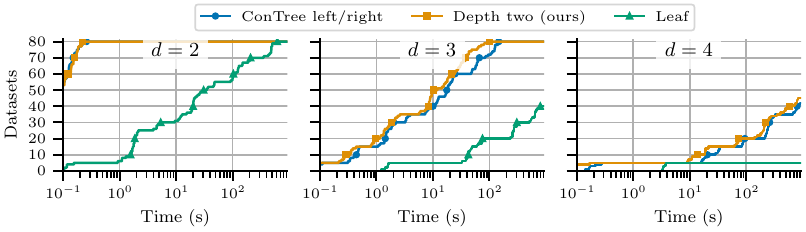}
        \caption{Runtime (s) performance of the three terminal cases.}
    \label{fig:terminal_runtime}
    \end{subfigure}

    \begin{subfigure}{\textwidth}
        \centering
        \includegraphics[width=\textwidth]{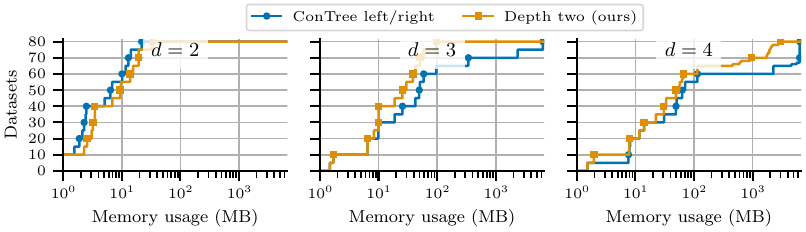}
        \caption{Memory usage (MB) (s) for the two terminal cases. The leaf case is left out, since it frequently does not finish the search.}
    \label{fig:terminal_memory}
    \end{subfigure}
\label{fig:terminal_results}
\caption{Comparison of the three terminal cases for CODT (single thread). Both the left-right and our depth-two terminal case outperform the leaf terminal case by several orders of magnitude. Our depth-two terminal case has a slight edge over the left-right terminal case in both memory and runtime.}
\end{figure*}
\begin{figure}[t!]
    \centering
    \includegraphics[width=\columnwidth]{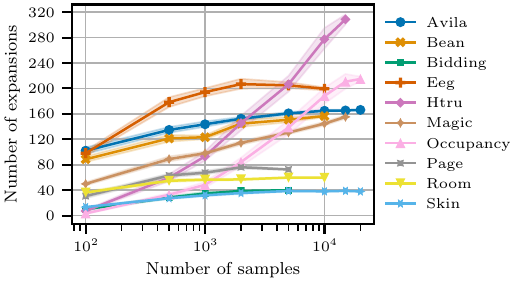}
    \caption{Scalability of solving depth-two trees for all classification datasets with at least $5000$ samples. Note the logarithmic horizontal axis. The number of expansions seems to be a logarithmic function of the number of samples.}
    \label{fig:d2_ablation_samples}
\end{figure}

Because \contree and \textsc{Quant-BnB} are previously shown to be orders of magnitude faster than the MIP approaches \textsc{OCT} \citep{bertsimas2017optimal} and \rsoct \citep{hua2022odt_mip_branching}, and the SAT approach by \citet{shati2023sat_odt_nonbin}, we do not compare with these methods. 

\section{Additional Experiment Results}
\label{app:more_experiments}
This appendix provides further experiment results. 
First, we experiment with what terminal case to use. 
Second, we provide a comparison of all \numstrategies search strategies.
Third and finally, we provide additional experiment results for comparing \codt with the baselines.

\subsection{Terminal Case}
By default, the terminal case of the search tree is the leaf node case, i.e., when the remaining depth limit reaches zero. 
However, as described in Appendix~\ref{app:implementation}, the use of a larger terminal case in the search tree is beneficial for both runtime and memory usage. For example, \citet{demirovic2022murtree} introduce a specialized terminal case for subproblems with a depth limit of two, provided the input features are all binary. \citet{manual_brita2025contree} provide a specialized terminal case for when the depth limit is two, but the root node split of the subproblem is already decided. Both show significant improvements in memory usage and runtime. Therefore, we also experiment with what terminal case to use in the search tree.

We test the following two questions.
First, what is the impact of replacing the base case from a leaf node to a larger subtree?
Second, how does our terminal case scale with an increasing number of samples? 

\paragraph{The impact of changing the terminal case.}

We compare three different terminal cases for the search tree:
\begin{description}
    \item[Leaf] The terminal case in the search tree is a leaf node, i.e., the remaining depth budget $d$ is zero. In our implementation, we keep track of the cost tuple when splitting the data (see Appendix~\ref{app:implementation}), and therefore computing the runtime complexity of the leaf terminal case is $\mathcal{O}(1)$ (assuming the search node is already constructed).
    \item[Left/right] The terminal case in the search tree is an AND node with state $\langle \mathcal{D}, d, f, \tau \rangle$ with $d$ equal to two. That is, the depth-two subproblem for which the split feature $f$ and the split threshold $\tau$ of the root node split is fixed. This is the terminal case that \citet{manual_brita2025contree} propose. We call it the \emph{left/right} terminal case, because it simultaneously solves the left and right subtree of a depth-two tree. The runtime complexity of this procedure is $\mathcal{O}(|\mathcal{D}||\mathcal{F}|)$.
    \item[Depth two] The terminal case in the search tree is an OR node with state $\langle \mathcal{D}, d \rangle$ with $d$ equal to two, i.e., a full depth-two subproblem. This is a new terminal case which we propose and describe in Appendix~\ref{app:implementation}. We solve it by using the main CODT algorithm (Alg.~\ref{alg:main} with a custom depth-two search strategy, which calls the \emph{left-right} terminal case for each AND node). In the first two experiments in this subsection, we fix the depth-two search strategy to \textsc{DFS-Blossom}. Below, we experiment with which depth-two search strategy is empirically best.
\end{description}

We run the same experiment as in the main text. For each dataset, we run \codt with the three terminal cases five times, for depth $d = 2, 3, 4$. We run this experiment without multi-threading.

Fig.~\ref{fig:terminal_runtime} reports the cumulative runtime distribution for the three terminal cases. These results show that both the left/right terminal case and the depth-two terminal case are significantly faster than using the leaf terminal case. For $d=2$, the depth-two terminal case has a reduces the runtime compared to the leaf terminal case by  $99.7\%$ (geometric mean reduction), while the left/right and depth-two terminal case perform approximately similar. For $d\geq3$, we see similar speedups.
This huge improvement was also reported by \citet{manual_brita2025contree} and is explained by the fact that optimizing the final left and right branching decisions can be done in linear time by the left/right terminal case, which is a huge improvement compared to splitting the data for each possible split, as would be done in the leaf terminal case.

Our proposal to turn depth-two OR nodes into the terminal case has a slight edge in runtime over the left/right terminal case for $d\geq 3$. For $d=3$, we observe an average runtime reduction of $32\%$, and for $d=4$, a reduction of $28\%$ (geometric mean, averaged over all runs where both cases finished before the time-out).

Fig.~\ref{fig:terminal_memory} shows that our depth-two terminal case has a positive impact on the memory usage as well. For $d=3$, we record an average memory usage reduction of $60\%$, and for $d=4$, a reduction of $48\%$ (again, computed with the geometric mean, averaged over all runs where both cases finished before the time-out).

Therefore, in all our experiments, we use the depth-two terminal case as the default terminal case.

\paragraph{Scalability of the depth-two terminal case.}
To solve depth-two trees, in the worst-case we need to call the left/right procedure one time for each possible split for each feature, resulting in a total worst-case runtime of $\mathcal{O}(|\mathcal{D}|^2|\mathcal{F}|^2)$. However, in practice we observe that we solve depth-two problems much faster than quadratic time. 

To test this, we set up an experiment where, for each of the datasets with at least $5000$ samples, we randomly sample an increasing amount of samples and measure the number of expansions required to solve the optimal depth-two decision tree problem. An expansion here is the single execution of the left/right procedure (Alg.~\ref{alg:d1}), which runs in $\mathcal{O}(|\mathcal{D}||\mathcal{F}|)$ time. Since we sample randomly, we repeat the experiment $100$ times and report the average number of expansions.

\begin{figure*}[t!]
    \centering
    \includegraphics[width=\textwidth]{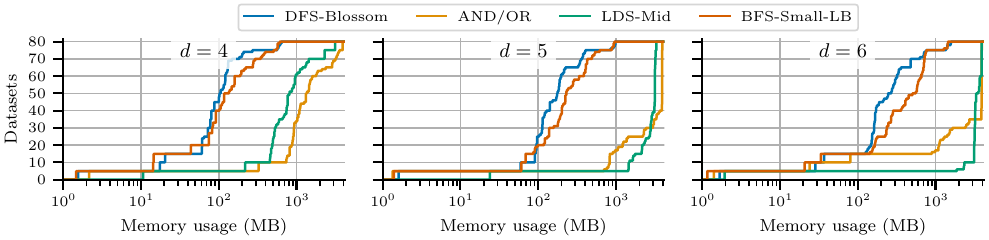}
    \caption{Cumulative distribution of memory usage (MB) for the four search strategy categories.}
    \label{fig:ecdf_ss_memory}
\end{figure*}
\begin{figure}[t!]
    \centering
    \includegraphics[width=\columnwidth]{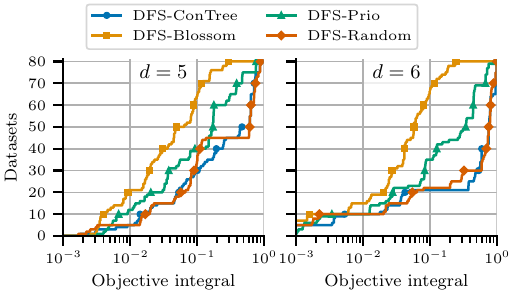}
    \caption{Cumulative distribution of the objective integral (OI) for the depth-first search strategies. The ordering by \contree, simulated by \textsc{DFS-ConTree} performs similar to a random ordering. Both orderings proposed by us, \textsc{DFS-Prio} and \textsc{DFS-Blossom} perform significantly better.}
    \label{fig:oi_ecdf_dfs}
\end{figure}

Fig.~\ref{fig:d2_ablation_samples} shows that the number of expansions required to solve optimal depth-two trees seems to scale as a logarithmic function of the number of samples in the dataset. If the number of expansions indeed is a logarithmic function of the samples, then the true worst-case runtime might be $\mathcal{O}(|\mathcal{F}|^2|\mathcal{D}| \log |\mathcal{D}|)$. As of yet, however, we do not know if this bound is correct, and leave this as future work to prove.

\subsection{Search Strategies}
\label{app:search_strategies_experiments}
In the experiments below, we review the variants for three of the four search strategy categories: DFS, BFS, and LDS.

\paragraph{Depth-first search.}

\begin{figure*}[p!]
    \centering
    \begin{subfigure}{\textwidth}
        \includegraphics[width=\textwidth]{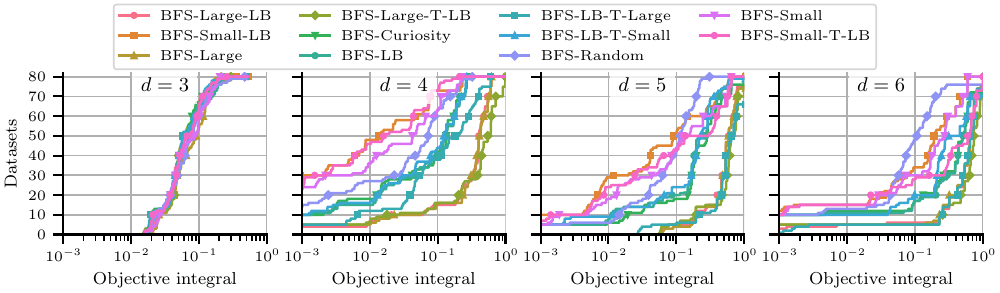}
        \caption{Cumulative distribution of the OI for the best-first search strategies. For $d=4$ and $d=5$, the heuristics that prioritize expanding search nodes with small support perform the best. For $d=6$, the random heuristic emerges as the best approach.}
        \label{fig:bfs_oi_ecdf}
    \end{subfigure}
    
    \begin{subfigure}{\textwidth}
    \centering
        \includegraphics[width=\columnwidth]{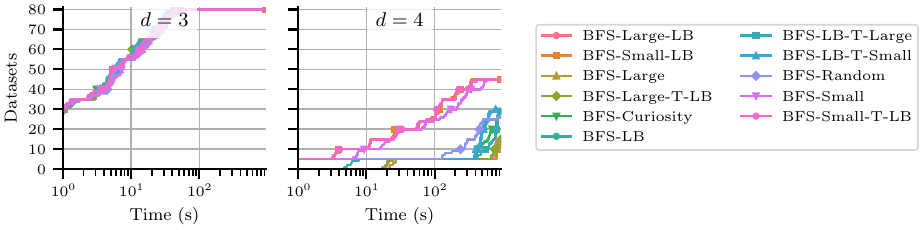}
        \caption{Cumulative distribution of runtime (s) for best-first search methods. For $d=3$ we observe no significant differences. For $d=4$, the heuristics that focus on search nodes with small support perform best, while the remaining heuristics perform worse than random.}
        \label{fig:bfs_runtime}
    \end{subfigure}
    \label{fig:bfs_results}
    \caption{Results for the best-first search strategies.}
\end{figure*}
\begin{figure*}[p!]
    \begin{subfigure}{0.48\textwidth}
        %note the ratio here is computed based on the ratios of the figures, so that they appear the same size!
        \includegraphics[width=.969689455\textwidth]{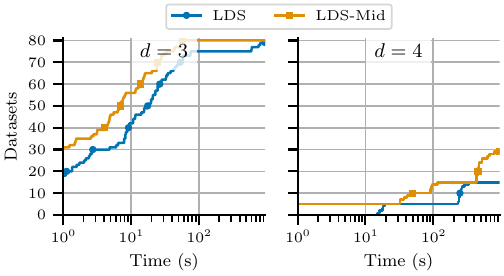}
        \caption{Cumulative runtime (s) distribution for the limited discrepancy search strategies for $d=3,4$.}
        \label{fig:lds_runtime}
    \end{subfigure}
    \hfill
    \begin{subfigure}{0.48\textwidth}
        \includegraphics[width=\textwidth]{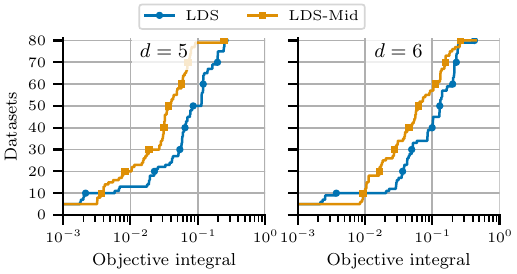}
        \caption{Cumulative distribution of the objective integral (OI) for the limited discrepancy search strategies for $d=5,6$.}
        \label{fig:lds_oi}
    \end{subfigure}
    \caption{For LDS, for both runtime and anytime performance (objective integral), picking the midpoint of the threshold intervals is better than picking the split with the lowest discrepancy.}
    \label{fig:lds}
\end{figure*}
\begin{figure*}[t!]
    \centering
    \includegraphics[width=\textwidth]{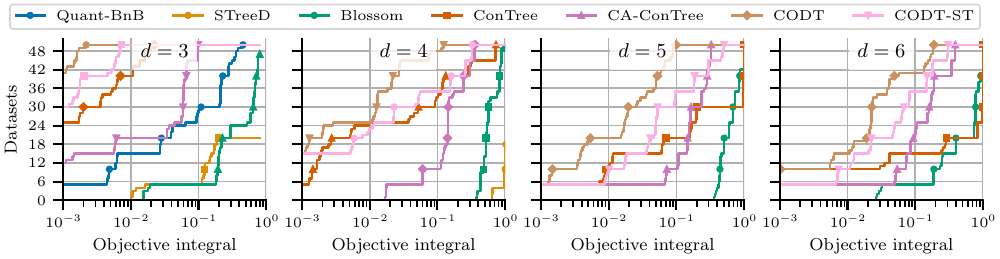}
    \caption{Cumulative distribution of the objective integral (OI) for all baselines methods including \blossom for only the binary classification datasets.}
    \label{fig:oi_ecdf_blossom}
\end{figure*}

We compare the four DFS strategies described in Appendix~\ref{app:extended_search_strategies}.
For runtime performance, we see small differences and therefore they are not reported here. For depth three and four we also see small differences, and therefore we limit the result reported here to training trees of depth five and six.

Fig.~\ref{fig:oi_ecdf_dfs} shows the anytime performance for the four DFS strategies, which reveals that the search order of \contree, which we mimic in the strategy \textsc{DFS-ConTree} performs similar to a DFS strategy that chooses the next node to expand randomly.
As is also the case for \contree, both \textsc{DFS-ConTree} and \textsc{DFS-Random} the anytime performance gets worse for larger depth limits.

The two new DFS approaches that we propose, \textsc{DFS-Prio}, which prioritizes expanding nodes with a low lower bound, and \textsc{DFS-Blossom} which balances the expansion of left and right subproblems, perform significantly better.
It is particularly interesting to see how such a small change to the search strategy, i.e., balancing left and right subproblem expansion, can have such an impact on the performance. 

Based on these results, we use \textsc{DFS-Blossom} as the DFS approach in the main text.

\paragraph{Best-first search.}
Fig.~\ref{fig:bfs_oi_ecdf} and~\ref{fig:bfs_runtime} compare the ten BFS heuristics presented in Table~\ref{tab:bfs_heuristics}, and as a baseline check, we include \textsc{BFS-Random} which assigns at a random heuristic value to a search node when it is constructed.

Interestingly, only three of the heuristics have better performance (both runtime and OI) than the random heuristic: \textsc{BFS-Small-LB}, \textsc{BFS-Small} and \textsc{BFS-Small-T-LB}. This clearly shows that heuristics that prioritize nodes with small support should be prioritized. The only heuristic that prioritizes small nodes but is not better than the random heuristic is \textsc{BFS-LB-T-Small}, which prioritizes nodes with a low lower bound and uses small support as a tie breaker. Given that \textsc{BFS-Small-T-LB}, which prioritizes nodes with small support and uses the lower bound as a tie breaker does perform better than the random heuristic, we can clearly attribute the performance to prioritizing nodes with small support. Similarly, heuristics that prioritize nodes with large support all perform significantly worse than the random heuristic. 

Surprisingly, for $d=6$, the random heuristic outperforms all other heuristics in anytime performance. Apparently, for larger depths some random expansions help diversify the search and improve the anytime performance.

Based on these results, we use \textsc{BFS-Small-LB} as the BFS approach in the main text

\paragraph{Limited discrepancy search.}
Figs.~\ref{fig:lds_runtime} and~\ref{fig:lds_oi} compare two variants of LDS, both based on \cacontree \citep{kiossou2026anytime_codt}. The only difference is how the split threshold $\tau_i$ is chosen. \textsc{LDS} chooses the split threshold with the lowest discrepancy, whereas $\textsc{LDS-Mid}$ follows the default in this paper: it chooses the mid point, i.e., $i = \lfloor (a+b)/2\rfloor$. Fig.~\ref{fig:lds_runtime} shows that for the smaller depth limits, \textsc{LDS-Mid} outperforms \textsc{LDS} in runtime. Similarly, Fig.~\ref{fig:lds_oi} shows \textsc{LDS-Mid} is also better for anytime performance. Therefore, we use \textsc{LDS-Mid} as the LDS method in the main text.

\paragraph{Memory usage.}
\label{app:memory}
Fig.~\ref{fig:ecdf_ss_memory} shows the distribution of the memory usage for the four search strategy categories. DFS, as expected, has the lowest memory usage. \textsc{BFS-Small-LB} follows shortly after, which shows that this BFS variant is relatively similar to DFS. Since it prioritizes nodes with low support it adopts a search order similar to DFS and closes nodes with low support for expanding new nodes with larger support. The other BFS strategies that do not prioritize search nodes with low support use much more memory. This is also the case for LDS and AND-OR. These strategies open many new search nodes but fail to close the nodes, resulting in a search tree that is too large for memory.

\paragraph{Optimization profiles.}
Figs.~\ref{fig:anytime_plot_search_strategies}-\ref{fig:anytime_plot_search_strategies_3} shows the optimization profiles of the four search strategy categories for twelve of the datasets (we left out the simple ones). They show that \textsc{DFS-Blossom} and \textsc{LDS-Mid} often have the best anytime performance, while \textsc{BFS-Small-LB} occasionally also does well.

For only a few cases among those plotted, do we see lower bounds are closing the optimization gap. For example, for EEG $d=4$, \textsc{DFS-Blossom} starts to close the gap close to the time-out, and for HTRU $d=4$, both \textsc{BFS-Small-LB} and \textsc{DFS-Blossom} close the optimization gap before the time out.
For larger depth limits, informative lower bounds are hard to obtain, which may explain the poor performance of \textsc{AND-OR} and BFS which use the lower bounds as the primary guiding heuristic.

\subsection{Further Comparison with Baselines}
\label{app:exp_other_baselines}

\begin{figure}[t!]
    \centering
    \includegraphics[width=\columnwidth]{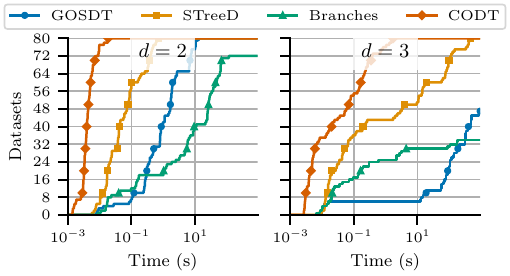}
    \caption{Cumulative runtime (s) distribution on the down-sampled classification datasets (100 random samples).}
    \label{fig:runtime_small_others}
\end{figure}
\paragraph{Comparison with \blossom.}
Since the available implementation for \blossom only supports binary classification, we did not include it in the comparison in the main text, because not all datasets in our classification benchmark are binary classification datasets. Therefore, we report a separate comparison here on only the binary classification datasets (see Table~\ref{tab:data}).
\blossom did not find and prove the optimal solution for any of the benchmarks we rain, therefore we here only discuss its anytime performance. 

Fig.~\ref{fig:oi_ecdf_blossom} shows the anytime performance of \blossom and the other methods on the binary classification datasets. It performs worse than all the methods which are designed specifically for handling numeric features. It only outperforms \streed, which just as \blossom, assumes the input data is binarized. 

The results for \codt in this plot are obtained with the search strategy \textsc{DFS-Blossom}, the DFS strategy inspired by \blossom, which shows that the core idea of the \blossom algorithm, to balance expansion of left and right subtrees translates well to our setting with numeric features.

\paragraph{Comparison with \gosdt, \streed, and \branches.}
For $d=3$ \gosdt runs into the 32GB memory limit for each dataset and for $d=2$, it can solve three datasets within the time and memory limit. \branches runs into the 32GB memory limit for $d=2$ for every dataset. In the main text, we also saw that \streed struggles with memory and runtime. Therefore, we here provide another comparison on subsampled classification datasets.

For each classification datasets, we sample five smaller datasets of 100 samples, under the condition that there are at least two distinct classes. We then compare the baselines \gosdt, \streed, and \branches with \codt for both $d=2$ and $d=3$.

The results are shown in Fig.~\ref{fig:runtime_small_others}. \codt (single threaded) is several orders of magnitude faster than both \gosdt and \branches and one order of magnitude faster than \streed.

\paragraph{Result table.}
The figures in the main text only provide a summary overview of the results. For completeness, we report the runtime results for \codt and the baselines in Table~\ref{tab:runtime_others}.

\begin{table}[t!]
    \centering
    \small
    \setlength\tabcolsep{3pt}
    \begin{tabular}{lccccc}
\toprule

Dataset & Depth & Quant-BnB & ConTree & CODT-ST & CODT \\ \midrule

%%%%%%%%%%%%%%%%%%%%%%%%%%%%%%%%%%%%%%%%%%
% TTOP/TTOS table
%%%%%%%%%%%%%%%%%%%%%%%%%%%%%%%%%%%%%%%%%%
Avila & 3 &
- & % TTOS: 900 (avila d=3) quantbnb
41 & % TTOS: 25 (avila d=3) contree
77 & % TTOS: 51 (avila d=3) codt-1t
\textbf{14} \\ % TTOS: 10 (avila d=3) codt
% Skipping avila d=4, all SS trivial or unsolved
Bank & 3 &
2 & % TTOS: 2 (bank d=3) quantbnb
\textbf{\textless 1} & % TTOS: 0 (bank d=3) contree
\textbf{\textless 1} & % TTOS: 0 (bank d=3) codt-1t
\textbf{\textless 1} \\ % TTOS: 0 (bank d=3) codt
Bean & 3 &
- & % TTOS: 614 (bean d=3) quantbnb
61 & % TTOS: 57 (bean d=3) contree
62 & % TTOS: 62 (bean d=3) codt-1t
\textbf{12} \\ % TTOS: 10 (bean d=3) codt
% Skipping bean d=4, all SS trivial or unsolved
Bidding & 3 &
19 & % TTOS: 19 (bidding d=3) quantbnb
3 & % TTOS: 3 (bidding d=3) contree
\textbf{\textless 1} & % TTOS: 0 (bidding d=3) codt-1t
\textbf{\textless 1} \\ % TTOS: 0 (bidding d=3) codt
& 4 &
- & % Missing result (bidding d=4) quantbnb
28 & % TTOS: 26 (bidding d=4) contree
50 & % TTOS: 3 (bidding d=4) codt-1t
\textbf{17} \\ % TTOS: 0 (bidding d=4) codt
Eeg & 3 &
- & % TTOS: 900 (eeg d=3) quantbnb
148 & % TTOS: 118 (eeg d=3) contree
155 & % TTOS: 119 (eeg d=3) codt-1t
\textbf{24} \\ % TTOS: 20 (eeg d=3) codt
% Skipping eeg d=4, all SS trivial or unsolved
Fault & 3 &
- & % TTOS: 900 (fault d=3) quantbnb
52 & % TTOS: 46 (fault d=3) contree
29 & % TTOS: 4 (fault d=3) codt-1t
\textbf{6} \\ % TTOS: 3 (fault d=3) codt
% Skipping fault d=4, all SS trivial or unsolved
Htru & 3 &
- & % TTOS: 900 (htru d=3) quantbnb
173 & % TTOS: 119 (htru d=3) contree
196 & % TTOS: 150 (htru d=3) codt-1t
\textbf{36} \\ % TTOS: 29 (htru d=3) codt
% Skipping htru d=4, all SS trivial or unsolved
Magic & 3 &
- & % TTOS: 862 (magic d=3) quantbnb
93 & % TTOS: 39 (magic d=3) contree
120 & % TTOS: 119 (magic d=3) codt-1t
\textbf{26} \\ % TTOS: 26 (magic d=3) codt
% Skipping magic d=4, all SS trivial or unsolved
Occupancy & 3 &
- & % TTOS: 900 (occupancy d=3) quantbnb
24 & % TTOS: 15 (occupancy d=3) contree
19 & % TTOS: 11 (occupancy d=3) codt-1t
\textbf{5} \\ % TTOS: 2 (occupancy d=3) codt
% Skipping occupancy d=4, all SS trivial or unsolved
Page & 3 &
627 & % TTOS: 627 (page d=3) quantbnb
4 & % TTOS: 2 (page d=3) contree
5 & % TTOS: 2 (page d=3) codt-1t
\textbf{1} \\ % TTOS: 1 (page d=3) codt
& 4 &
- & % Missing result (page d=4) quantbnb
\textbf{633} & % TTOS: 339 (page d=4) contree
- & % TTOS: 900 (page d=4) codt-1t
727 \\ % TTOS: 244 (page d=4) codt
Raisin & 3 &
61 & % TTOS: 61 (raisin d=3) quantbnb
\textbf{\textless 1} & % TTOS: 0 (raisin d=3) contree
2 & % TTOS: 2 (raisin d=3) codt-1t
\textbf{\textless 1} \\ % TTOS: 0 (raisin d=3) codt
& 4 &
- & % Missing result (raisin d=4) quantbnb
\textbf{98} & % TTOS: 11 (raisin d=4) contree
- & % TTOS: 900 (raisin d=4) codt-1t
259 \\ % TTOS: 1 (raisin d=4) codt
Rice & 3 &
- & % TTOS: 900 (rice d=3) quantbnb
21 & % TTOS: 6 (rice d=3) contree
34 & % TTOS: 8 (rice d=3) codt-1t
\textbf{6} \\ % TTOS: 1 (rice d=3) codt
% Skipping rice d=4, all SS trivial or unsolved
Room & 3 &
- & % TTOS: 900 (room d=3) quantbnb
1 & % TTOS: 0 (room d=3) contree
3 & % TTOS: 1 (room d=3) codt-1t
\textbf{\textless 1} \\ % TTOS: 0 (room d=3) codt
& 4 &
- & % Missing result (room d=4) quantbnb
\textbf{62} & % TTOS: 16 (room d=4) contree
478 & % TTOS: 16 (room d=4) codt-1t
64 \\ % TTOS: 23 (room d=4) codt
Segment & 3 &
70 & % TTOS: 70 (segment d=3) quantbnb
2 & % TTOS: 1 (segment d=3) contree
5 & % TTOS: 4 (segment d=3) codt-1t
\textbf{\textless 1} \\ % TTOS: 1 (segment d=3) codt
& 4 &
- & % Missing result (segment d=4) quantbnb
\textbf{297} & % TTOS: 159 (segment d=4) contree
- & % TTOS: 900 (segment d=4) codt-1t
516 \\ % TTOS: 82 (segment d=4) codt
Skin & 3 &
284 & % TTOS: 283 (skin d=3) quantbnb
8 & % TTOS: 7 (skin d=3) contree
12 & % TTOS: 10 (skin d=3) codt-1t
\textbf{4} \\ % TTOS: 3 (skin d=3) codt
& 4 &
- & % Missing result (skin d=4) quantbnb
\textbf{170} & % TTOS: 35 (skin d=4) contree
- & % TTOS: 461 (skin d=4) codt-1t
230 \\ % TTOS: 119 (skin d=4) codt
Wilt & 3 &
18 & % TTOS: 18 (wilt d=3) quantbnb
\textbf{\textless 1} & % TTOS: 0 (wilt d=3) contree
\textbf{\textless 1} & % TTOS: 0 (wilt d=3) codt-1t
\textbf{\textless 1} \\ % TTOS: 0 (wilt d=3) codt
& 4 &
- & % Missing result (wilt d=4) quantbnb
21 & % TTOS: 9 (wilt d=4) contree
63 & % TTOS: 63 (wilt d=4) codt-1t
\textbf{16} \\ % TTOS: 16 (wilt d=4) codt
\midrule
Ratio &
 & 60.3
 & 2.3
 & 3.9
 & 1.0
\\

\bottomrule
\end{tabular}
\caption{Runtime (s) for our method \codt and the baselines \quantbnb and \contree for the classification datasets. Dataset-depth combinations where either all methods are done within one second or none finish within the time-out (15 min.) are left out. `-' indicates time-out. Both \blossom and \cacontree did not finish for any dataset within the time limit. The reported ratio is the mean geometric ratio of runtime compared to \codt.}
\label{tab:runtime_others}

\end{table}

\begin{figure*}[p]
\centering
\begin{subfigure}{\textwidth}
    \centering
    \includegraphics[width=\linewidth]{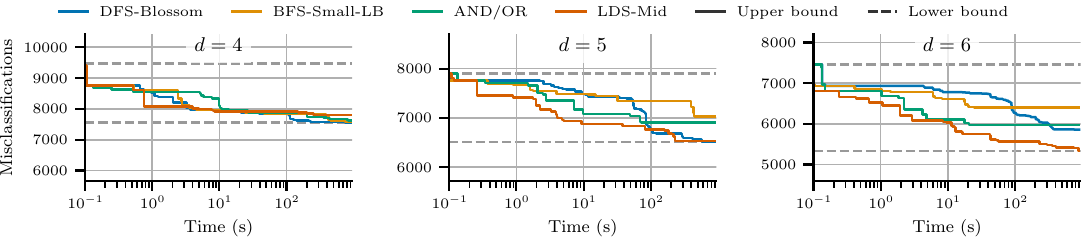}
    \caption{Avila}
    \label{fig:anytime_avila}
\end{subfigure}

\begin{subfigure}{\textwidth}
    \centering 
    \includegraphics[width=\linewidth]{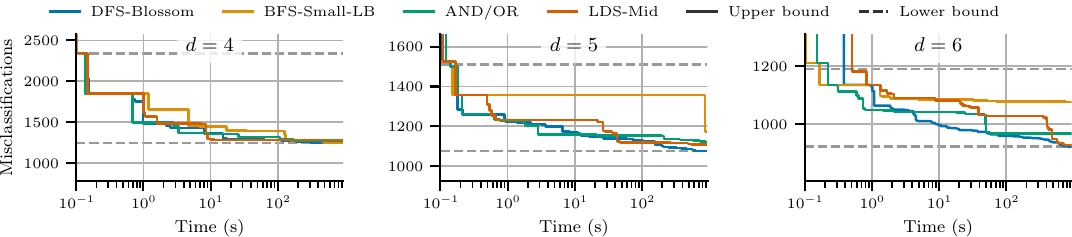} 
    \caption{Bean} 
    \label{fig:anytime_bean} 
\end{subfigure}

\begin{subfigure}{\textwidth} 
    \centering 
    \includegraphics[width=\linewidth]{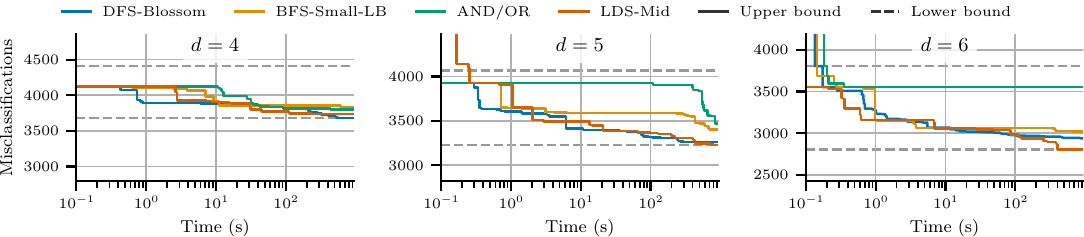} 
    \caption{EEG} 
    \label{fig:anytime_eeg} 
\end{subfigure} 

\begin{subfigure}{\textwidth} 
    \centering 
    \includegraphics[width=\linewidth]{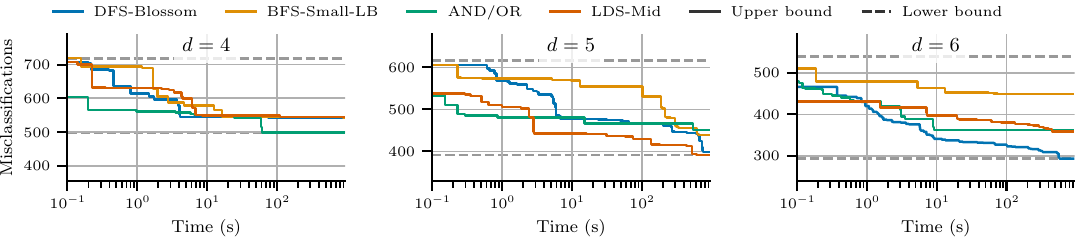} 
    \caption{Fault} 
    \label{fig:anytime_fault} 
\end{subfigure} 

\caption{Upper bound (incumbent) and lower bounds for the four search strategy categories for several datasets. The top gray dashed line indicates the initial CART solution. The bottom gray dashed line indicates the optimal or best solution found by any method.} 
\label{fig:anytime_plot_search_strategies} 
\end{figure*}

\begin{figure*}[p]
\centering
\begin{subfigure}{\textwidth} 
    \centering 
    \includegraphics[width=\linewidth]{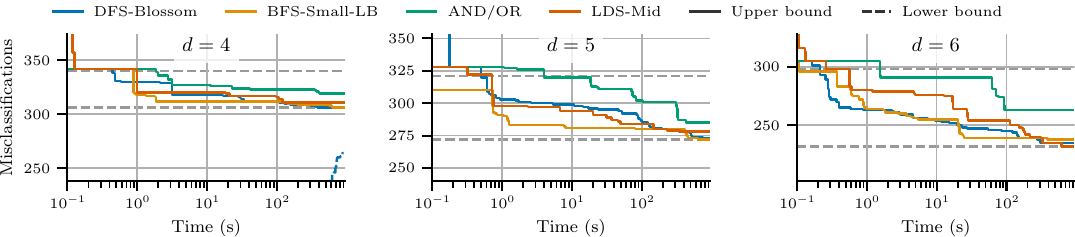} 
    \caption{HTRU} 
    \label{fig:anytime_htru} 
\end{subfigure} 

\begin{subfigure}{\textwidth}
    \centering
    \includegraphics[width=\linewidth]{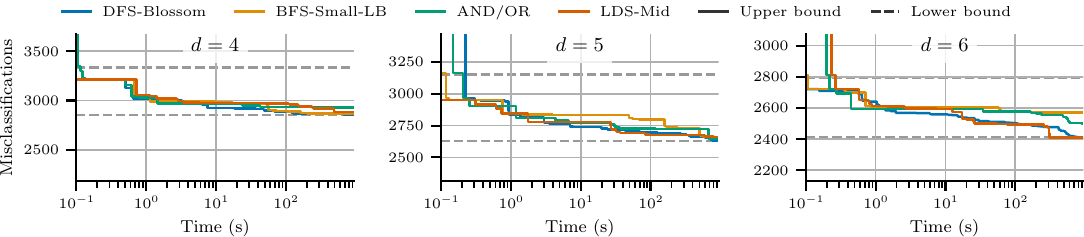}
    \caption{Magic}
    \label{fig:anytime_magic}
\end{subfigure}

\begin{subfigure}{\textwidth}
    \centering 
    \includegraphics[width=\linewidth]{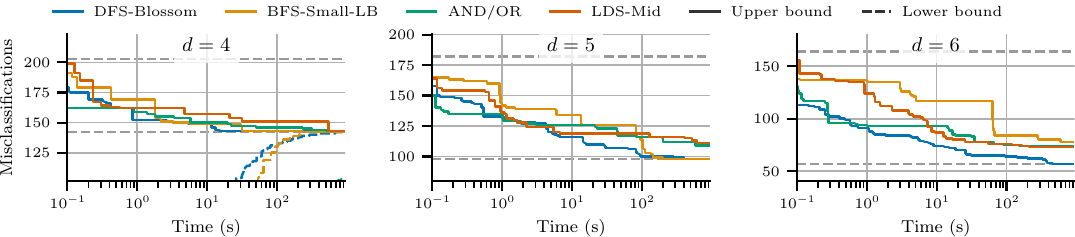} 
    \caption{Occupancy} 
    \label{fig:anytime_occupancy} 
\end{subfigure}

\begin{subfigure}{\textwidth} 
    \centering 
    \includegraphics[width=\linewidth]{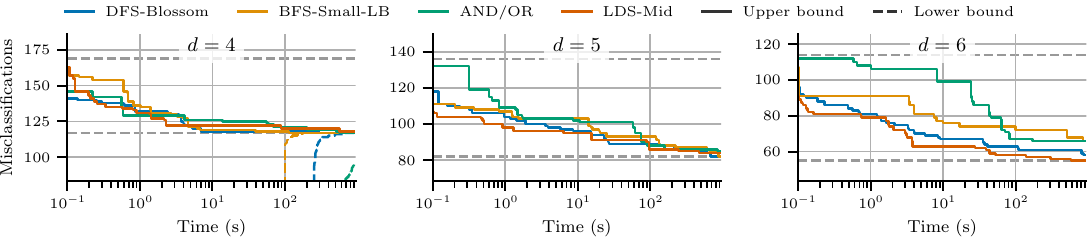} 
    \caption{Page} 
    \label{fig:anytime_page} 
\end{subfigure} 

\caption{Upper bound (incumbent) and lower bounds for the four search strategy categories for several datasets. The top gray dashed line indicates the initial CART solution. The bottom gray dashed line indicates the optimal or best solution found by any method.} 
\label{fig:anytime_plot_search_strategies_2} 
\end{figure*}

\begin{figure*}[p]
\centering
\begin{subfigure}{\textwidth} 
    \centering 
    \includegraphics[width=\linewidth]{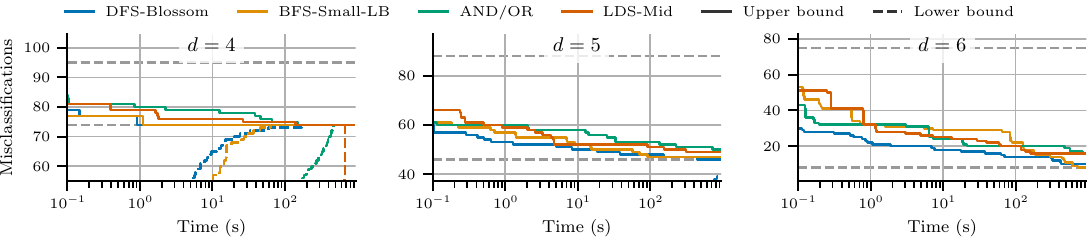} 
    \caption{Raisin} 
    \label{fig:anytime_raisin} 
\end{subfigure} 

\begin{subfigure}{\textwidth}
    \centering
    \includegraphics[width=\linewidth]{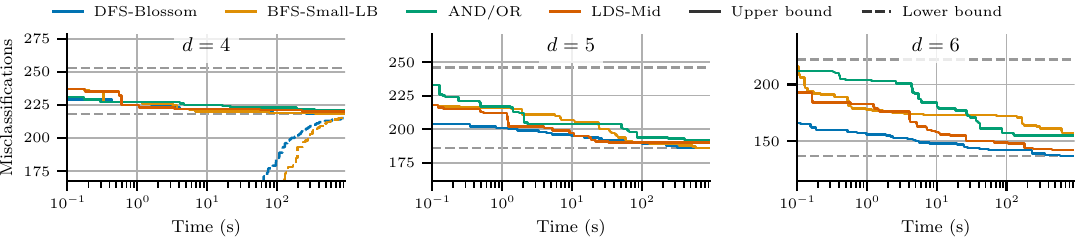}
    \caption{Rice}
    \label{fig:anytime_rice}
\end{subfigure}

\begin{subfigure}{\textwidth}
    \centering 
    \includegraphics[width=\linewidth]{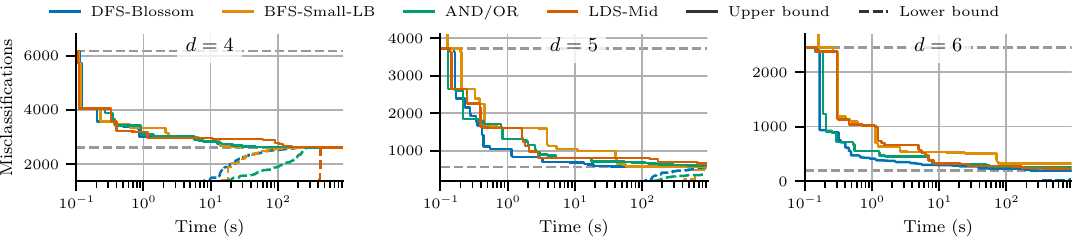} 
    \caption{Skin} 
    \label{fig:anytime_skin} 
\end{subfigure}

\begin{subfigure}{\textwidth} 
    \centering 
    \includegraphics[width=\linewidth]{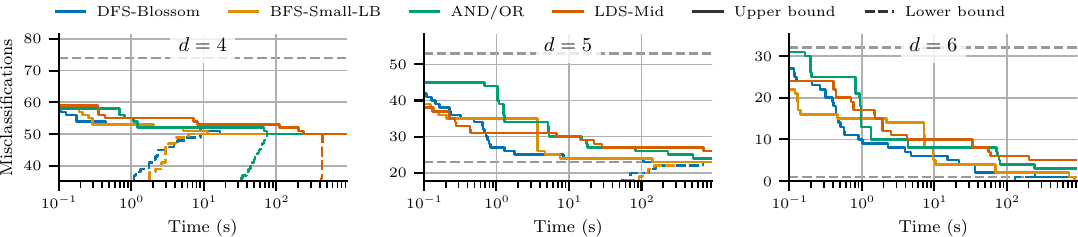} 
    \caption{Wilt} 
    \label{fig:anytime_wilt} 
\end{subfigure} 

\caption{Upper bound (incumbent) and lower bounds for the four search strategy categories for several datasets. The top gray dashed line indicates the initial CART solution. The bottom gray dashed line indicates the optimal or best solution found by any method.} 
\label{fig:anytime_plot_search_strategies_3} 
\end{figure*}

\end{document}